\documentclass[authoryear,preprint,review,12pt]{elsarticle}

\usepackage{geometry}
\usepackage{amssymb}
\usepackage{amsmath}
\usepackage{tabularx}
\usepackage{tabu}
\usepackage{amsfonts, amsthm}
\usepackage[hidelinks]{hyperref}
\usepackage{algorithm}
\usepackage{algpseudocode}

\usepackage{textcomp}
\usepackage{stfloats}
\usepackage{graphicx}

\usepackage{subcaption}
\usepackage{url}
\usepackage{verbatim}
\usepackage{multicol}
\usepackage{multirow} % For multirow functionality
\usepackage{array}    % For fixed column widths (optional)

\usepackage{balance}

\usepackage[T1]{fontenc}
\usepackage[utf8]{inputenc}

\usepackage{booktabs} % For better table formatting
\usepackage[table]{xcolor}
\definecolor{darkred}{RGB}{139,0,0}

\usepackage{threeparttable}
\usepackage{comment} % comment out blocks
\newtheoremstyle{italictheorem} % Style name
{3pt} % Space above
{3pt} % Space below
{\itshape} % Italic body font
{} % Indent amount (empty = no indent, \parindent = paragraph indent)
{\bfseries} % Theorem head font
{.} % Punctuation after theorem head
{ } % Space after theorem head
{} % Theorem head spec
\theoremstyle{italictheorem}
\newtheorem{definition}{Definition}
\newtheorem{lemma}{Lemma}

\newtheorem{theorem}{Theorem}

\newtheorem{proposition}{Proposition}

\DeclareSymbolFont{boldoperators}{OT1}{cmr}{bx}{n}
\SetSymbolFont{boldoperators}{bold}{OT1}{cmr}{bx}{n}
\edef\bar{\unexpanded{\protect\mathaccentV{bar}}\number\symboldoperators16}

\usepackage{setspace}
\usepackage{natbib}
\setcitestyle{numbers,square}
\journal{Pattern Recogintion}

\newcolumntype{Y}{>{\centering\arraybackslash}X}

\begin{document}

\begin{frontmatter}



\title{Random Feature Pruning for Multi-Class Classification}


\author[a]{Mohammad Amin Omidi} %% Author name
\ead{mamin.omidi@shahed.ac.ir}
\author[a]{Babak Seyfe\corref{cor1}} 
\ead{seyfe@shahed.ac.ir}
\author[b]{Shahrokh Valaee} 
\ead{valaee@ece.utoronto.ca}
\affiliation[a]{organization={Information Theoretic Learning Systems Laboratory (ITLSL), Department of Electrical Engineering,  Shahed University},%Department and Organization
            %addressline={}, 
            city={Tehran},
            %postcode={}, 
            %state={Tehran},
            country={Iran}}
\affiliation[b]{organization={Department of Electrical and Computer Engineering, University of Toronto},%Department and Organization
	%addressline={}, 
	city={Toronto},
	%postcode={}, 
	%state={},
	country={Canada}}

\cortext[cor1]{Corresponding author} % Optional: Marks the corresponding author

%% Abstract
\begin{abstract}
Weighted sums of random kitchen sinks (W-RKS) is a commonly used feature extraction technique for supervised learning tasks. A major drawback of this method is the redundancy and ineffectiveness of many extracted random features. This paper investigates efficient methods for reducing such redundancy in multi-class classification problems. We specifically formulate feature selection as a variant of the sequentially nested model selection problem and introduce the Successive Empirical Excess Risk Difference Maximization (SEED-MAX) algorithm. It recursively computes empirical risks to identify the most parsimonious yet informative set of features. We derive an analytical lower bound for the probability of correct feature selection by the SEED-MAX algorithm. This lower bound on the probability of correct feature selection can be made arbitrarily close to one by increasing the training dataset size. 
To further enhance the rate of pruning, we evaluate the features retained by SEED-MAX using the Fisher test (F-Test). Empirical evaluations on UCR time-series classification datasets show that the sequential application of SEED-MAX with the Fisher test achieves approximately 92 percent rate of pruning with negligible degradation in classification accuracy. Implementations of the SEED-MAX and SEED-MAX-F-Test methods are available at \url{https://github.com/OmidiAmin/seed-max-ftest}.
\end{abstract}

%\includegraphics{figs/Graph_Abs.png}
%\includegraphics[width=\textwidth]{figs/Graph_Abs.pdf}

%\item A novel formulation of random features pruning as  nested model selection.  
%\item We Rank the features according to the trained weights to limit the search space.
%\item We Propose a new feature selection method called SEED-MAX.
%\item We establish a lower bound on probability of correct feature selection by SEED-MAX.

%\item For further pruning the redundant features, we apply Fisher test to SEED-MAX.
%\item SEED-MAX achieve the highest rate of pruning compared to state-of-the-arts methods.

%% Keywords

\end{frontmatter}



\section{Introduction}
\label{sec:intro}
%% Labels are used to cross-reference an item using \ref command.
{A} trending feature extraction technique for supervised learning tasks is \emph{weighted sums of random kitchen sinks }(W-RKS) \cite{rahimi2008weighted}, which is a generalization to \emph{random Fourier features }(RFF) \cite{rahimi2007random} and a successor to \emph{extreme learning machine} (ELM) \cite{huang2015trends,Lin2015extreme,nair2026approximate}. This method has led to the creation of many practical algorithms \cite{li2019towards,li2021towards,dempster2020rocket,dempster2021minirocket,tan2021multirocket,dempster2023hydra,bier2022variable,schlegel2022hdc,hao2025hit}. 

Instead of extracting features through an optimization process, one can generate a large number of features by random sampling (from an appropriately specified distribution). Then, the regression weights for these randomly generated features are estimated from the training data, and the weighted sum of these random features is used for the label prediction. By deploying a sufficiently large number of random samples, it has been demonstrated both theoretically \cite{rahimi2008weighted,rahimi2007random,li2019towards,li2021towards,rudi2017generalization,bach2017equivalence,alaoui2015fast,chen2022concentration,wang2023overparameterized,liao2018spectrum,yu2025regularized,wang2026Generalization,nguyen2026random,kumar2026scalable} and empirically \cite{rahimi2008weighted,rahimi2007random,dempster2020rocket,dempster2021minirocket,tan2021multirocket} that the effectiveness of feature extraction using random sampling is comparable to that achieved through training with a particular infinite-dimensional feature map within a \emph{reproducing kernel Hilbert space} (RKHS). Moreover, since randomization is computationally cheaper than optimization, W-RKS has low computational complexity in its training phase \cite{rahimi2008weighted}. As a result, learning with W-RKS is a kind of approximation to learning with kernel methods \cite{liu2021random}. 

%Although W-RKS uses random sampling to simplify the learning problem, there is a main drawback regarding this type of learning. 
%In general, because the sampling distribution is irrelevant to the underlying distribution of the observations, there is a possibility that a large portion of the generated random features is redundant for the label prediction. This issue has been the focus of some recent studies.
Although W-RKS simplifies learning via random sampling, its sampling distribution is decoupled from the data distribution, which can lead to many redundant random features that do not contribute to label prediction. This limitation has been highlighted in recent studies \cite{li2021towards,salehinejad2022s,erocket,maximin,keshavarzian2024hiervar,chen2024pocket,Uribarri2024DetachROCKET}.

A way to reduce the redundancy of random sampling is to employ importance sampling \cite{li2019towards,li2021towards}. However, a key limitation of this approach is that the importance sampling distribution is computed solely from the input samples in the training data, without leveraging the relationship between inputs and labels.

%In \cite{li2019towards} and \cite{li2021towards}, the authors have studied a kind of importance sampling based on the leverage score to reduce the minimum required number of random features. It guarantees a fast learning rate based on the order of the excess risk upper bound. However, there is a primary issue regarding their proposed approach. The computation of the leverage score is based merely on the input samples in the training data, and the leverage score does not use the input-label relationship in its importance sampling distribution.

One can implement a supervised feature selection method to address the above issue, which retains only the features that significantly contribute to label prediction
\cite{erocket, maximin,keshavarzian2024hiervar,chen2024pocket,Uribarri2024DetachROCKET}.

In this regard, in \cite{erocket} a fast feature selection method is introduced, termed \emph{enhanced ROCKET }(E-ROCKET), to reduce the computational complexity of learning with \emph{random convolutional kernel transform }(ROCKET) \cite{dempster2020rocket,dempster2021minirocket} features while empirically preserving the probability of correct time-series classification. More recently, Keshavarzian and Valaee \cite{keshavarzian2024hiervar} applied Fisher test \cite{fisher1921studies, fisher1970statistical} as a post-processing step to E-ROCKET, achieving approximately a threefold improvement in the rate of pruning. They referred to their method as \emph{hierarchical variable selection }(HIERVAR) \cite{keshavarzian2024hiervar}.  
Other supervised feature selection methods for pruning random features have also been proposed in \cite{maximin,chen2024pocket,Uribarri2024DetachROCKET}, which rely on heuristic sparse optimization techniques to remove redundant features. Among these, \emph{pruning ROCKET} (POCKET) \cite{chen2024pocket} is an efficient implementation of a widely used sparse learning optimization algorithm known as the \emph{alternating direction method of multipliers} (ADMM), originally introduced in \cite{boyd2011distributed} and later applied to deep neural network compression in \cite{lee2022dnn}. Detach-ROCKET \cite{Uribarri2024DetachROCKET} is another state-of-the-art method based on the \emph{backward elimination} technique \cite{Hastie2009elements}, which is a simple and widely adopted sparse optimization approach.
%Experimental studies of the above-mentioned feature selection methods on the publicly available UCR time-series classification database \cite{dau2019ucr} indicate that HIERVAR \cite{keshavarzian2024hiervar} achieves the best performance in terms of both the rate of pruning and the probability of correct classification. 

All the aforementioned methods have empirically investigated feature selection to reduce the redundancy of random features, yet they lack a learning guarantee. Additionally, a critical question remains unanswered in the literature: To what extent can random features be pruned in a multi-class classification task, particularly when the number of classes is large, while still ensuring accurate label prediction?

We have addressed the above question in the following way: First, we have formulated feature selection as a problem of selecting a model from a sequence of nested models. Subsequently, we have proposed a method for estimating the most parsimonious model that utilizes the smallest number of features while achieving the minimum error. Our feature selection method approximates the most parsimonious model by identifying the change point in the trend of empirical risk minima for the nested sequence of models. We demonstrate, both empirically and theoretically, that the proposed method retains the subset of informative features with high probability.

In summary, the main contributions of this paper are as follows:

\begin{itemize} 
	\item We formulate feature selection as a special case of model selection over a sequence of nested models. By ranking the randomly generated features according to the norms of the trained regression weight vectors, the feature selection process is significantly simplified, and the search space is reduced to a discrete interval whose length is equal to the total number of random features.
	
	\item We propose a new consistent feature selection method (see \eqref{eq:change point detection criterion} and Algorithm \ref{alg:cpd}) for W-RKS, called \emph{successive empirical excess risk difference maximization }(SEED-MAX), that estimates the smallest possible number of features (over sorted features) needed to guarantee non-significant degradation in the prediction performance of the labels. We establish a lower bound on the probability of correctly estimating the minimum number of random features required to achieve the global minimum risk over sequentially nested models (Theorem \ref{thm:pac lower bound}). 
	\item Empirical feature sorting based on trained regression weights can retain redundant collinear features. To address this, we introduce a post-processing step in SEED-MAX that employs the Fisher test. The resulting combined method, SEED-MAX-F-Test, effectively reduces feature sorting errors and yields a more refined feature set.
	
	\item Our experimental results on real data (Table \ref{tab:ucr-main}, Fig. \ref{fig:ucr perf plane all data}, Fig. \ref{fig:ucr perf plane specific}, and Fig. \ref{fig:cd-diag}) and synthetic data (Fig. 6 on page 6 of the Supplementary Materials) demonstrate that, in scenarios with a large number of classes, our proposed methods achieve the highest pruning rates among state-of-the-art approaches, without significant degradation in the probability of correct classification. Moreover, through a synthetic experiment (see Fig. \ref{fig:sdmax analyitic lower}), we compare the lower bound in \eqref{eq:pac cpd lower bound} from Theorem \ref{thm:pac lower bound} with the exact probability of correctly detecting the true underlying pruned model. This figure shows that the derived lower bound closely tracks the exact probability of identifying the true underlying pruned model and exhibits very low sensitivity to variations in the training dataset size.

\end{itemize}

The remainder of this paper is organized as follows: Section \ref{sec: rks} provides a concise overview of the W-RKS method and its primary challenges. Section \ref{sec: ftr-slctn nstd-mdl-slctn} formulates the feature selection problem as a special case of sequentially nested model selection. In Section \ref{sec:seedmax}, we introduce SEED-MAX, our consistent and practical feature selection method. Section \ref{sec:seedmax analysis} presents the theoretical analysis of the SEED-MAX algorithm. In Section \ref{sec:Further pruning via Fisher Test}, we introduce the use of the Fisher test as a post-processing step for the SEED-MAX algorithm, resulting in a variant referred to as the SEED-MAX-F-Test. Sections \ref{sec:exp results} and \ref{sec:conclusion} are dedicated to empirical studies and conclusions, respectively. 
A detailed proof of our main analytical result concerning  SEED-MAX (Theorem \ref{thm:pac lower bound}) is provided in \ref{sec:main-thm-proof}.
A detailed explanation of multi-class classification using random sampling, along with additional evaluations of SEED-MAX on a synthetic dataset, are provided in Supplementary Materials.
%The detailed proofs of the analytical results, along with additional experiments on a simulated dataset, are provided in the appendices.

\subsection{Notations}
We use bold italic lowercase letters to represent vectors and bold italic uppercase letters for matrices. For a real matrix \(\boldsymbol{W} \in \mathbb{R}^{J \times K}\), where \(J\) and \(K\) are natural numbers, \(\boldsymbol{w}_{k}\) denote the vector corresponding to the \(k\)-th column of the matrix $\boldsymbol{W}$, respectively. The symbol \([K]\) represents the set \(\{1,\dots,K\}\), where \(K\) is a natural number.  
Calligraphic letters (e.g., \(\mathcal{F}\), \(\mathcal{G}\), and $\mathcal{J}$) are used to represent the sets or families of functions. 
The symbol $\lVert \boldsymbol{W} \rVert_F$ denotes the Frobenious norm of the matrix $\boldsymbol{W}$.  The set of non-negative real numbers is denoted by \(\mathbb{R}^+\). 
The \(l_p\) norm (or semi-norm) of  \(\boldsymbol{w} \in \mathbb{R}^K\) is defined as \(\lVert \boldsymbol{w} \rVert_p = \left( \sum_{k=1}^K |w_k|^p \right)^{\frac{1}{p}}\), where \(p > 0\). The symbol \( \lfloor \cdot \rfloor \) denotes the floor operation, defined as $\lfloor x \rfloor := \max\{ m \in \mathbb{Z} | m \le x \}$, where $\mathbb{Z}$ represents the set of all integers. 

\section{Weighted Sums of Random Kitchen Sinks}\label{sec: rks}
\noindent
At first, we review W-RKS as one of the generalized methods for learning with random features. Let us define the training dataset \(\mathcal{D} = \{(\boldsymbol{x}_n, \boldsymbol{y}_n)\}_{n=1}^{N}\), where each \(M\)-dimensional input vector \(\boldsymbol{x}_n \in \mathcal{X}\), \(n=1, \dots, N\), is sampled from the input domain \(\mathcal{X} \subseteq \mathbb{R}^M\) and paired with a \(J\)-dimensional vector-valued label\footnote{For multi-class classification problems, vector-valued labels can be constructed using one-hot encoding. A more detailed discussion on the treatment of multi-class classification, including how the main analytical result in \cite{rahimi2008weighted} for binary classification can be extended to the multi-class setting, is provided in Section 1 of Supplementary Materials.} \(\boldsymbol{y}_n \in \mathcal{Y}\), where \(\mathcal{Y} \subseteq \mathbb{R}^J\).

We assume that every input-label pair $(\boldsymbol{x}, \boldsymbol{y})$ in the training data is drawn independently from an unknown and fixed distribution $P(\boldsymbol{x}, \boldsymbol{y})$. Our goal is to create a mapping from inputs to labels using a hypothesis function \(\boldsymbol{f}: \mathcal{X} \rightarrow \mathcal{Y}\) that accurately predicts the labels. To accomplish this, the W-RKS method initially samples \( K \) random \( I \)-dimensional parameter vectors, denoted as \( \boldsymbol{\theta}_1, \dots, \boldsymbol{\theta}_K \in \mathbb{T} \), where \(\mathbb{T} \subseteq \mathbb{R}^I\). These vectors are drawn independently and identically from a specified distribution \( Q(\boldsymbol{\theta}) \). For simplicity, these sampled parameters are collectively represented as \(\boldsymbol{S} := [\boldsymbol{\theta}_1, \dots, \boldsymbol{\theta}_K]\). Subsequently, using a feature function \(\phi(\cdot, \cdot) : \mathcal{X} \times \mathbb{T} \rightarrow [-1, 1]\), the hypothesis function \(\boldsymbol{f}\) is defined as \(\boldsymbol{f}(\boldsymbol{x}) = \boldsymbol{W} \boldsymbol{\phi}(\boldsymbol{x}, \boldsymbol{S})\), for all \(\boldsymbol{x} \in \mathcal{X}\), where $\boldsymbol{W}\in \mathbb{R}^{J\times K}$ is the weight matrix associated with $\boldsymbol{f}$, and
\begin{equation}\label{eq:feature vector}
	\boldsymbol{\phi}(\boldsymbol{x},\boldsymbol{S}) := [\phi(\boldsymbol{x},\boldsymbol{\theta}_1), \dots, \phi(\boldsymbol{x},\boldsymbol{\theta}_K)]^T.
\end{equation}
We define the family of hypothesis functions $\mathcal{F}_k$, $k\in [K]$, as follows
\begin{align}\label{eq:F_k}
	\mathcal{F}_{k} := \bigg\{ \boldsymbol{f}   \bigg| 
	~ \boldsymbol{f}(\boldsymbol{x}) = \sum_{i=1}^{k} \boldsymbol{w}_{i} \phi(\boldsymbol{x}, \boldsymbol{\theta}_{u_i}), ~\boldsymbol{x}\in \mathcal{X},
	 \boldsymbol{w}_i \in \mathbb{R}^{J}, u_i \in [K], i\in[k] \bigg\}.
\end{align}
Based on the definition in \eqref{eq:F_k}, $\mathcal{F}_{k}$ contains all functions with less than or equal to $k$ randomly generated features.
The risk minimizer function \(\boldsymbol{f}^*_k(\boldsymbol{x})\) over the W-RKS family \(\mathcal{F}_k\) is defined as follows \cite{rahimi2008weighted,li2019towards,li2021towards}
\begin{equation}\label{eq:risk minimizer optimization}
	\boldsymbol{f}^*_k = \operatorname*{argmin}_{\boldsymbol{f}\in \mathcal{F}_{ k}} R[\boldsymbol f],
\end{equation}
where
\begin{equation}\label{eq:risk function}
	R[\boldsymbol{f}] = \int_{\mathcal{X}\times \mathcal{Y}} l(\boldsymbol{f}(\boldsymbol{x}),\boldsymbol{y})\text{d}P(\boldsymbol{x},\boldsymbol{y}),
\end{equation}
is the risk functional, $l:\mathcal{X}\times\mathcal{Y}\rightarrow \mathbb{R}^+$ is the loss function, and $P(\boldsymbol{x}, \boldsymbol{y})$ represents the fixed but unknown distribution of the input and label vector pairs, respectively. 
Since the data distribution \(P(\boldsymbol{x}, \boldsymbol{y})\) is unknown, the risk functional \(R[\boldsymbol{f}]\) in \eqref{eq:risk minimizer optimization} is generally not accessible. However, it can be approximated by the empirical risk functional \(\hat{R}[\boldsymbol{f}]\), defined as follows \cite{rahimi2008weighted,li2019towards,li2021towards}
\begin{equation}\label{eq:empirical risk}
	\hat{R}[\boldsymbol{f}] :=\frac{1}{N}\sum_{n=1}^{N}l(\boldsymbol{f}(\boldsymbol{x}_n), \boldsymbol{y}_n). 
\end{equation}

By replacing the risk functional \( R[\boldsymbol{f}] \) in \eqref{eq:risk minimizer optimization} with the empirical risk functional \( \hat{R}[\boldsymbol{f}] \), and introducing a regularization term alongside the loss function to control computational complexity, we obtain \cite{rahimi2008weighted,li2019towards,li2021towards}
\begin{equation}\label{eq:regularized erm 2nd}
	\hat{\boldsymbol{f}}_k = \operatorname*{argmin}_{\boldsymbol{f} \in \mathcal{F}_{ k}} \left\{\hat{R}[\boldsymbol{f}] +\frac{\lambda}{N} \lVert \boldsymbol{W} \rVert_F^2\right\},
\end{equation}
where $\lambda > 0$ is the regularization constant and $\boldsymbol{f}(\boldsymbol{x}) = \boldsymbol{W} \boldsymbol{\phi}(\boldsymbol{x},\boldsymbol{S})$. For inference tasks, it is reasonable to expect that the empirical risk minimizer $\hat{\boldsymbol{f}}_k$ in \eqref{eq:regularized erm 2nd} provides a close approximation to the true risk minimizer $\boldsymbol{f}^*_k$ in \eqref{eq:risk minimizer optimization}. For this to hold, it is necessary that the empirical risk $\hat{R}[\boldsymbol{f}]$ be a close approximation of the risk $R[\boldsymbol{f}]$, for every $\boldsymbol{f} \in \mathcal{F}$. The following definition introduces a class of loss functions for which this property holds as the training dataset size increases.
\begin{definition}[Glivenko-Cantelli (GC) Class of Loss Functions \cite{hajiani2024model,wainwright2019high}]\label{def:Glivenko-Cantelli}
	The class of loss functions $\mathcal{L}= \left\{ l(\boldsymbol{f}(\cdot),\cdot) \middle| \boldsymbol{f} \in \mathcal{F} \right\}$ over the domain $\mathcal{X}\times \mathcal{Y}$ with a  given distribution $P(\boldsymbol{x},\boldsymbol{y})$ is said to be Glivenko-Cantelli class if for every $\epsilon > 0$ the following holds
	\begin{equation}\label{eq:glivenko-cantelli}
		\lim_{N\rightarrow \infty} \mathbb{P}\left[ \sup_{\boldsymbol{f}\in \mathcal{F}} \left| \frac{1}{N} \sum_{n=1}^N l(\boldsymbol{f}(\boldsymbol{x}_n), \boldsymbol{y}_n) - R[\boldsymbol{f}] \right| > \epsilon \right] = 0. 
	\end{equation}
\end{definition}
For every bounded loss function (e.g., the squared error loss with bounded inputs), Hoeffding’s inequality (see Lemma 4.5 in \cite{ShalevShwartz2014Understanding}) holds, which is a specific case of the Glivenko–Cantelli condition in Definition \ref{def:Glivenko-Cantelli}. 

The condition in \eqref{eq:glivenko-cantelli} can be expressed as an equivalent $\epsilon$-$\delta$ definition of the limit \cite{wainwright2019high}. For every \(\delta \in (0,1]\), there exists an upper bound \(\epsilon(N, \delta)\) of the order \(o(1)\) and a positive integer \(N_1\) such that for all \(N \geq N_1\), the following inequality holds
\begin{align}\label{eq:epsilon-delta bound}  
	\mathbb{P}\left[ \sup_{\boldsymbol{f} \in \mathcal{F}} \left| \hat{R}[\boldsymbol{f}] - R[\boldsymbol{f}] \right| > \epsilon(N, \delta) \right] \leq \delta.  
\end{align}

% Furthermore, the function \(\hat{\boldsymbol{f}}^{(\lambda)}_k\) can be expressed as \(\hat{\boldsymbol{f}}^{(\lambda)}_k(\boldsymbol{x}) = \hat{\boldsymbol{W}}^{(\lambda, k)} \boldsymbol{\phi}(\boldsymbol{x}, \boldsymbol{S})\), where \(\hat{\boldsymbol{W}}^{(\lambda, k)} \in \mathbb{R}^{J \times K}\) represents its corresponding weight matrix containing at most \(k\) non-zero columns.

Due to the lower computational complexity of randomization than optimization, W-RKS is a fast method \cite{rahimi2008weighted}, and from learning guarantees in \cite{rahimi2008weighted,li2019towards,li2021towards}, with a sufficiently large number of random samples $K$ and dataset size $N$, the solution in \eqref{eq:regularized erm 2nd} will be close to the risk minimizer in \eqref{eq:risk minimizer optimization}. 

Since the features in \eqref{eq:feature vector} are generated randomly, it is likely that many of these $K$ random features do not significantly contribute to the label prediction. Thus, we can safely eliminate these redundant features without notably affecting performance. Given $K$ random features, we investigate how we can reduce the number of these features to $\hat{k}<K$ such that $\hat{k}$ is the minimum possible value that preserves the label prediction performance. 
This problem can be cast as a special case of \emph{nested model selection} \cite{hajiani2024model} which has some practically and theoretically useful properties.

\section{Feature Selection Through Sequentially Nested Model Selection}\label{sec: ftr-slctn nstd-mdl-slctn}
This section presents the feature selection problem for W-RKS as a specific instance of nested model selection. 
Before proceeding, it is essential to review certain definitions and theorems from \cite{hajiani2024model}.

\begin{definition}[Nested Models \cite{hajiani2024model}]\label{def: nested model}
	Consider two families of hypothesis functions $\mathcal{F}_1$ and $\mathcal{F}_2$ with the input domain $\mathcal{X}$, whose learnable parameters are from domains $\Omega_1$ and $\Omega_2$, respectively. The family $\mathcal{F}_1$ is said \emph{to be nested} in $\mathcal{F}_2$, denoted by $\mathcal{F}_1 \subseteq \mathcal{F}_2$, if and only if  for every $\boldsymbol{f}_1 \in \mathcal{F}_1$ with the associated learnable parameter $\boldsymbol{W}_1 \in \Omega_1$, there is a function $\boldsymbol{f}_2 \in \mathcal{F}_2$ with the associated learnable parameter $\boldsymbol{W}_2 \in \Omega_2$ such that for all $\boldsymbol{x}\in \mathcal{X}$, $\boldsymbol{f}_1(\boldsymbol{x}) = \boldsymbol{f}_2(\boldsymbol{x}) $. Moreover, a model class $\{\mathcal{F}_k\}_{k=1}^K$ is said to be \emph{sequentially nested class}, denoted by $\mathcal{F}_1 \subseteq \mathcal{F}_2 \subseteq \dots \subseteq \mathcal{F}_K$, if for $K > 2$, every $i, j \in [K]$, with $i\le j$, $\mathcal{F}_i$ is nested in $\mathcal{F}_j$.
\end{definition}

Based on the definition of the sequentially nested class, the following useful property holds.
\begin{theorem}[From \cite{hajiani2024model}]\label{thm:elbow risk}
	
	For the sequentially nested model class $\{\mathcal{F}_k\}_{k=1}^K$ with respective risk minimizers $\{\boldsymbol{f}^*_k\}_{k=1}^K$, there is a model index $k^* \in [K]$ with the following property
	\begin{equation}\label{eq:risk elbow property}
		R[\boldsymbol{f}^*_1] \ge \dots \ge R[\boldsymbol{f}^*_{k^*-1}] >R[\boldsymbol{f}^*_{k^*}] = \dots = R[\boldsymbol{f}^*_{K}].  
	\end{equation}
\end{theorem}
\begin{proof}
	See the proof of Corollaray 3 in \cite{hajiani2024model}.
\end{proof}
Let us define the global minimum risk $R^*$ as follows
\begin{align}\label{eq:rstar}
	R^* := \min_{k \in [K]} R[\boldsymbol{f}^*_k]. 
\end{align}
According to \eqref{eq:risk elbow property}, \(k^*\) denotes the index of the \emph{most parsimonious family}\footnote{The most parsimonious family refers to the family of functions that includes the global risk minimizer while having the minimum number of learnable parameters (see page 100 of \cite{claeskens2008model}).} of functions where the global minimum risk \(R^*\) is achievable. Our goal is to create a method for estimating the model index \(k^*\)  based on the provided training data.
In practice, the true risk values are not accessible. Instead, we can use the sequence of empirical risks \(\{\hat{R}[\hat{\boldsymbol{f}}_k]\}_{k=1}^K\) to estimate the model index \(k^*\). As stated in Lemma 1 of \cite{hajiani2024model}, the sequence of empirical risks for sequentially nested classes is non-increasing. Consequently, the following relationships hold  
\begin{align}\label{eq:non-incrsng emp risks}  
	\hat{R}[\hat{\boldsymbol{f}}_1] \geq \hat{R}[\hat{\boldsymbol{f}}_2] \dots \geq \hat{R}[\hat{\boldsymbol{f}}_K].  
\end{align}
From the sequence of empirical risks in \eqref{eq:non-incrsng emp risks}, we aim to obtain a reliable estimate \(\hat{k}\)  of  \(k^*\),  the most parsimonious model index. To this end, we exploit the properties of nested models. Before proceeding, we demonstrate how nested model selection applies to the feature selection problem.

%It follows from  that for each \(k \in [K]\), we have \(\mathcal{F}_{k} \subseteq \mathcal{F}_{K}\). Moreover, for any \(k \in [K-1]\), it holds that \(\mathcal{F}_{k} \subseteq \mathcal{F}_{k+1}\). Consequently, the family of sets \(\{\mathcal{F}_{k}\}_{k=1}^K\) constitutes a sequentially nested class.
Based on Definition \ref{def: nested model}, the families of functions constructed according to \eqref{eq:F_k} are sequentially nested. However, 
there is a major challenge in feature selection over the families \(\{\mathcal{F}_{k}\}_{k=1}^K\), which is the necessity of evaluating each of the \(\binom{K}{k}\) possible combinations of \(k\) features during empirical risk minimization. This requirement leads to an exponential growth in computational complexity, rendering feature selection over the families \(\{\mathcal{F}_k\}_{k=1}^K\) computationally infeasible.
To address this challenge, we start by ranking the features in order of their importance, from the most to the least significant. We then construct a sequentially nested model class using this ranking to enable efficient feature selection. In the following, we elaborate on the method used for ranking the features.

Consider the hypothesis function \(\hat{\boldsymbol{f}}(\boldsymbol{x}) = \hat{\boldsymbol{W}} \boldsymbol{\phi}(\boldsymbol{x}, \boldsymbol{S})\), which is derived from the empirical risk minimization in \eqref{eq:regularized erm 2nd} over the family \(\mathcal{F}_K\). We denote $\boldsymbol{r} = [r_1,\dots, r_K]^T$ to be an ordering over the set $[K]$ in a way that the following property holds 
\begin{equation}\label{eq:emp sorting}
	\left\lVert \hat{\boldsymbol{w}}_{r_1} \right\rVert_p \ge  \left\lVert \hat{\boldsymbol{w}}_{r_2} \right\rVert_p \ge \dots 	\ge \left\lVert \hat{\boldsymbol{w}}_{r_K} \right\rVert_p,
\end{equation}
where $\hat{\boldsymbol{w}}_{r_k}$, $k\in [K]$ denotes the $r_k$-th column of $\hat{\boldsymbol{W}}$, which is assigned to the rank $k$ according to the ordering in \eqref{eq:emp sorting}.
According to this ordering, the feature \(\boldsymbol{\phi}(\cdot, \boldsymbol{\theta}_{r_1})\) is deemed to be the most significant, whereas \(\boldsymbol{\phi}(\cdot, \boldsymbol{\theta}_{r_K})\) is viewed as the least significant.
Based on the vector $\boldsymbol{r}$, we define the family $\mathcal{G}_k$ as follows
\begin{align}\label{eq:F_k_r_hat}
	\mathcal{G}_k := \bigg\{ \boldsymbol{g} \bigg| \boldsymbol{g}(\boldsymbol{x}) =\sum_{i=1}^{k} \boldsymbol{w}_{i} \phi(\boldsymbol{x}, \boldsymbol{\theta}_{r_i}), 
	~\boldsymbol{x}\in \mathcal{X},
	~  \boldsymbol{w}_i \in \mathbb{R}^{J}, i\in[k] \bigg\}.
\end{align}
It immediately follows from the above definition that the class \(\{\mathcal{G}_k\}_{k=1}^K\) is sequentially nested. 
Thanks to \eqref{eq:risk minimizer optimization} and \eqref{eq:regularized erm 2nd}, for each \( k \in [K] \), we define \(\boldsymbol{g}^*_k\) and \(\hat{\boldsymbol{g}}_k\) as the risk minimizer and (regularized) empirical risk minimizer solutions over the families \(\{\mathcal{G}_k\}_{k=1}^K\), respectively.
Considering that \(\{\mathcal{G}_k\}_{k=1}^K\) is a sequentially nested family of functions, the relationships in equations \eqref{eq:risk elbow property} and \eqref{eq:non-incrsng emp risks} also hold true for this family.   

In what follows, we demonstrate that some subsets of \(\{\mathcal{G}_k\}_{k=1}^K\) are also can be sequentially nested. This property, which is established in Proposition \ref{prp:nesting for sorted subsets}, will be effectively leveraged in the subsequent section to simplify the feature selection process.
\begin{proposition}[Nestedness over Sub-classes]\label{prp:nesting for sorted subsets}  
	Consider the sequentially nested class \(\{\mathcal{G}_k\}_{k=1}^{K}\), where \(\mathcal{G}_k\) is defined in \eqref{eq:F_k_r_hat} based on the ordering \(\boldsymbol{r}:=[r_1,\dots,r_K]^T\) that satisfies the property in \eqref{eq:emp sorting}. Let \(\boldsymbol{u}:=[u_1,\dots, u_{\tilde{K}}]^T\), where \(\tilde{K} \leq K\) is a fixed positive integer, and $u_k\in [K]$ such that $u_1 < u_2 <\dots < u_{\tilde{K}}$ holds. It then follows that the subsampled class \(\{\mathcal{G}_{u_k}\}_{k=1}^{\tilde{K}}\) also retains the sequentially nested property. 
\end{proposition}
\begin{proof} 
	From the definition in \eqref{eq:F_k_r_hat}, the assumption that the sequence $(u_k)_{k=1}^{\tilde{K}}$ is ascending, and the transitivity of the nesting operation (i.e., for every $k\in [K-2]$, $\mathcal{G}_{k}\subseteq \mathcal{G}_{k+1}$ and $\mathcal{G}_{k+1}\subseteq \mathcal{G}_{k+2}$ imply $\mathcal{G}_{k} \subseteq \mathcal{G}_{k+2}$), it immediately follows that
	$\mathcal{G}_{u_1} \subseteq \dots \subseteq \mathcal{G}_{u_{\tilde{K}}}$. As a result, $\{\mathcal{G}_{u_{k}}\}_{k=1}^{\tilde{K}}$ 
	forms a sequentially nested class.
\end{proof}
In the following sections, we demonstrate that the first-order and second-order differences of empirical risk sequences, which are defined below, are useful for estimating the number of informative features based on vector \(\boldsymbol{r}\).
\begin{definition}[Successive Empirical Excess Risk (SEER) Sequence \cite{hajiani2024model}]\label{def:seer}
	We refer to the sequence \(\{\Delta \hat{R}_k\}_{k=1}^{K-1}\) as the successive empirical excess risk (SEER) sequence, where each \(\Delta \hat{R}_k\) is defined as follows 
	\begin{align}\label{eq:seer}  
		\Delta \hat{R}_k := \hat{R}[\hat{\boldsymbol{g}}_k] - \hat{R}[\hat{\boldsymbol{g}}_{k+1}].  
	\end{align}  
\end{definition}
It will be demonstrated in the subsequent sections that the successive differences of SEER, defined below, play a crucial role in estimation of \(k^*\).
\begin{definition}[Successive Empirical Excess Risk Difference Sequences (SEEDS)]\label{def:seeds}
	We refer to the sequence \(\{\Delta^2 \hat{R}_k\}_{k=2}^{K-1}\) as the Successive Empirical Excess Risk Difference Sequence (SEEDS), where each \(\Delta^2 \hat{R}_k\) is defined as follows
	\begin{align}\label{eq:seeds}
		\Delta^2 \hat{R}_k := \Delta\hat{R}_{k-1} - \Delta\hat{R}_{k}.
	\end{align}
	
\end{definition}

%As assumed in the Glivenko-Cantelli definition, where the empirical risks approach to the true risk values in the limit, we can define \(\Delta R_k\) and \(\Delta^2 R_k\) as the asymptotic limits of \(\Delta \hat{R}_k\) and \(\Delta^2 \hat{R}_k\), respectively\footnote{For a detailed analysis, refer to the proof of Lemma 2 in \ref{sec:main-thm-proof}.}. Consequently, the excess risk $\Delta R_k$ is defined as follows
%\begin{align}\label{eq:delta rk}
%	\Delta R_k := R[\boldsymbol{g}^*_{k}] - R[\boldsymbol{g}^*_{k+1}].
%\end{align}
%Moreover, the excess risk difference $\Delta^2 R_k$ is defined as follows
%\begin{align}\label{eq:delta2 rk}
%	\Delta^2 R_k := \Delta R_{k-1} - \Delta R_{k}.
%\end{align}
Based on the established definitions and properties, we are ready to introduce our feature selection algorithm in the following section.

\section{Successive Empirical Excess Risk Difference Maximization} \label{sec:seedmax}
In the previous sections, we discussed feature extraction using W-RKS and illustrated how the feature selection problem can be reformulated as a specific case of sequentially nested model selection. Utilizing the characteristics of nested models and W-RKS features, we now introduce a method to estimate \(k^*\) within \([K]\), the index of the most parsimonious family that includes the global risk minimizer. 
As previously discussed, we constructed the sequentially nested families \(\{\mathcal{G}_k\}_{k=1}^K\) by incrementally adding features in accordance with the ordering established in \eqref{eq:emp sorting}. Based on \eqref{eq:risk elbow property}, \(k^*\) can be interpreted as a \emph{change point in the trend of the risk sequence} \(\{R[\boldsymbol{g}^*_k]\}_{k=1}^{K}\) over the sequentially nested class \(\{\mathcal{G}_k\}_{k=1}^K\). 

Given the model index \(k^*\), the sequence \(\small\{R[\boldsymbol{g}^*_k]\}_{k=1}^K\) can be divided into two distinct subsequences: \(\small\{R[\boldsymbol{g}^*_k]\}_{k=1}^{k^* -1}\) and \(\small\{R[\boldsymbol{g}^*_k]\}_{k=k^*}^{K}\).  
In the first subsequence, \(\small\{R[\boldsymbol{g}^*_k]\}_{k=1}^{k^* -1}\), the risk values are substantially larger than \(R^*\) which is defined in \eqref{eq:rstar}, and they exhibit a rapid decline as \(k\) increases. Conversely, in the second subsequence, \(\small\{R[\boldsymbol{g}^*_k]\}_{k=k^*}^{K}\), all the terms have converged to their minimum value, \(R^*\). Beyond \(k^*\), no further reduction in the risk values can be observed, reflecting the stabilization of the sequence at its optimal value.
However, the risk values \(\small\{R[\boldsymbol{g}^*_k]\}_{k=1}^{K}\) are not available, and we cannot locate $k^*$ based on the expressed trend over the sequence $\small\{R[\boldsymbol{g}^*_k]\}_{k=1}^{K}$. 

Alternatively, we can investigate the trend of the sequence $\{\hat{R}[\hat{\boldsymbol{g}}_k]\}_{k=1}^{K}$ for estimating $k^*$. Intuitively, based on the Glivenko-Cantelli property in \eqref{eq:glivenko-cantelli}, as the training sample size \(N\) increases, the empirical risk values \(\{\hat{R}[\hat{\boldsymbol{g}}_k]\}_{k=1}^{K}\) almost surely converge to their corresponding risks, expressed as\footnote{The intuition behind this will be formally justified in the proof of Theorem \ref{thm:pac lower bound}. In particular, refer to the PAC bound in \eqref{eq:erm gap} in \ref{sec:main-thm-proof}.}
\begin{align}\label{eq:almostsurely}
	\hat{R}[\hat{\boldsymbol{g}}_k] \xrightarrow[N\rightarrow\infty]{\text{a.s.}} R[\boldsymbol{g}^*_k], \quad \forall k \in [K],  
\end{align}
where \(\xrightarrow[N\rightarrow\infty]{\text{a.s.}}\) denotes almost sure convergence as $N$ goes to infinity. 
Consequently, we can estimate \(k^*\) by analyzing the behavior of the empirical risk sequence \(\{\hat{R}[\hat{\boldsymbol{g}}_k]\}_{k=1}^{K}\).
We aim to identify the model index \(\hat{k}\) that divides the sequence of empirical risk values \(\{\hat{R}[\hat{\boldsymbol{g}}_k]\}_{k=1}^{K}\) into two distinct subsequences \(\{\hat{R}[\hat{\boldsymbol{g}}_k]\}_{k=1}^{\hat{k}-1}\) and \(\{\hat{R}[\hat{\boldsymbol{g}}_k]\}_{k=\hat{k}}^{K}\), each exhibiting differing trends.

Based on the rate of reduction in the empirical risk, as quantified by \(\Delta^2 \hat{R}_k\)  in \eqref{eq:seeds}, we construct the SEEDS sequence \(\{\Delta^2 \hat{R}_k\}_{k=2}^{K-1}\). The model index corresponding to the maximum value of this sequence is then identified as the change point in the trend of the empirical risk sequence. This criterion is mathematically expressed as follows:
\begin{align}\label{eq:seed-max v1}
	\hat{k} := \operatorname{argmax}_{k \in [K-1]\setminus \{1\}} \Delta^2 \hat{R}_{k},
\end{align}
where it is assumed that the most parsimonious model index \(k^*\) lies within \(\{2,3,\dots, K-1\}\), aligning with the search space 
for the approximation \(\hat{k}\). While this optimization provides a structured approach to locate the change point \(\hat{k}\), it poses a significant computational challenge.
From the definition of \(\Delta^2\hat{R}_k\) in \eqref{eq:seeds}, determining \(\hat{k}\) requires computing three \emph{empirical risk minimization} (ERM) solutions, namely \(\hat{\boldsymbol{g}}_{k-1}\), \(\hat{\boldsymbol{g}}_{k}\), and \(\hat{\boldsymbol{g}}_{k+1}\), corresponding to the sequentially nested model families \(\mathcal{G}_{k-1}\), \(\mathcal{G}_{k}\), and \(\mathcal{G}_{k+1}\), respectively. Consequently, in the worst-case scenario, calculating the entire sequence \(\{\Delta^2 \hat{R}_k\}_{k=2}^{K-1}\) requires $K$ retrainings to compute the ERM solutions \(\{\hat{\boldsymbol{g}}_k\}_{k=1}^{K}\).

Since \(K\) often corresponds to a large number in practical applications, this results in a computationally prohibitive runtime, making the direct implementation of this optimization infeasible for large-scale problems.
To address the computational challenge associated with the optimization in \eqref{eq:seed-max v1}, we utilize the property described in Proposition \ref{prp:nesting for sorted subsets}, establishing that any subsample of sequentially nested families is also sequentially nested. Specifically, rather than comparing families that differ by only one or two features, we partition the range of model indices \([K]\) into (integer) \(s\) segments. Within this framework, comparisons are restricted to each segment's starting and ending points.

Limiting the computations to these segment boundaries decreases the number of retraining operations  \(K\) (in the worst case) to \(s+1\), where the integer $s > 0$ and  \(s \ll K\). This reduction significantly reduces the computational burden, making the optimization process feasible for large-scale applications.
Moreover, since the models being compared correspond to feature subsets that differ substantially (i.e., those at the segment boundaries), the differences in empirical risk values are more likely to be meaningful. This approach reduces the risk of erroneously identifying the change point \(\hat{k}\) due to marginal or negligible fluctuations in empirical risk values.

By leveraging this strategy, we achieve a more efficient and robust estimation of the change point \(\hat{k}\), addressing the limitations of the original optimization process.
Now, consider the sequentially nested class \(\{\mathcal{G}_k \}_{k=1}^K\). Let \(k_1\) and \(k_{s+1}\), where \(k_1 < k_{s+1}\), be two positive integers within the range \([K]\). To define the segment boundaries, we select \(s-1\) linearly spaced points \(\{k_i\}_{i=2}^{s}\) between \(k_1\) and \(k_{s+1}\), such that
\begin{align}\label{eq:sorted indices}
	k_1 < k_2 < \cdots < k_s < k_{s+1}, 
\end{align}
where the integer $s > 0$ is the number of discrete line segments created by the above partitioning.
Based on Proposition \ref{prp:nesting for sorted subsets}, the sequence of families \(\{\mathcal{G}_{k_i}\}_{i=1}^{s+1}\) is sequentially nested. Consequently, it possesses a most parsimonious model index \(k_{i^*} \in \{k_i\}_{i=1}^{s+1}\), which we aim to identify.

We claim that the following optimization yields a probably approximately correct (PAC) estimate of \(k_{i^*}\). Denoting this estimate by \(k_{\hat{i}}\) it is defined as follows:
\begin{equation}\label{eq:change point detection criterion}
	k_{\hat{i}} := \operatorname{argmax}_{k_i \in \{k_j\}_{j=1}^{s+1}} \Delta^2 \hat{R}_{k_i},
\end{equation}
where $\Delta^2 \hat{R}_{k_i} := \Delta \hat{R}_{k_{i-1}} - \Delta \hat{R}_{k_{i}}$, and $\Delta \hat{R}_{k_i} := \hat{R}[\hat{\boldsymbol{g}}_{k_i}] - \hat{R}[\hat{\boldsymbol{g}}_{k_{i+1}}]$. We refer to the optimization described above as \emph{Single-Layer Successive Empirical Excess Risk Difference Maximization (SL-SEED-MAX)}. As previously mentioned, the SL-SEED-MAX method involves at most \(s + 1\) retraining. However, this efficiency comes at a cost that \(k_{\hat{i}}\) may exceed \(\hat{k}\) as defined in \eqref{eq:seed-max v1}, which may result in retaining a larger number of features. Specifically, \(\hat{k}\) lies within the (updated) discrete search interval \((k_{\hat{i}-1}, k_{\hat{i}}]\). For the case where \( k_1 = 1 \) and \( k_{s+1} = K \), the length of this interval is approximately \( \lfloor \frac{K}{s} \rfloor \). For large values of \( K \), this length can be quite substantial, potentially leading to the retention of more redundant features.

\begin{figure}
	\centering
	\includegraphics[width=0.7\linewidth]{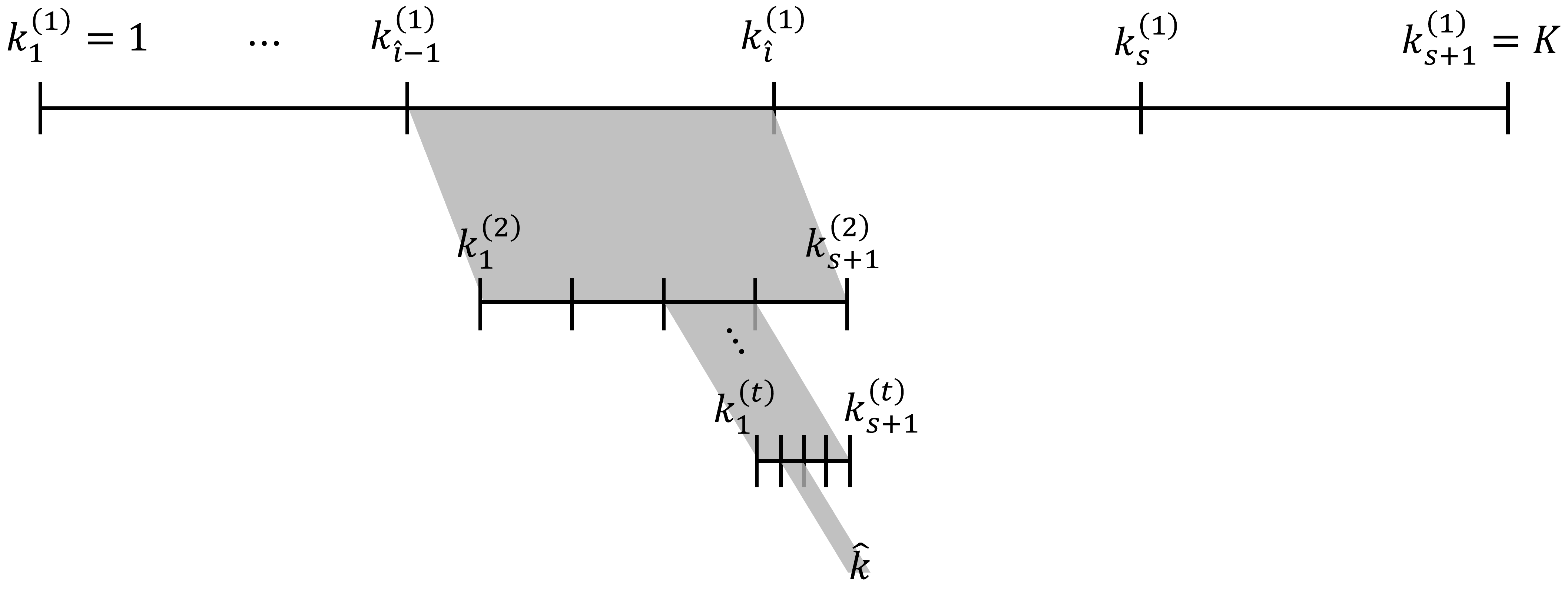}
	\caption{A visual illustration of the SEED-MAX algorithm (Algorithm \ref{alg:cpd}). 
	}
	\label{fig:seedmax visual illustration}
\end{figure}

To address this issue, as illustrated in Fig. \ref{fig:seedmax visual illustration}, the SL-SEED-MAX algorithm can be recursively applied to reduce the length of the feature selection search interval to a desirable small value. Algorithm \ref{alg:cpd} provides the details of this procedure.
At the first layer of feature selection, applying SL-SEED-MAX to the entire discrete search interval \([1, K]\) yields a smaller interval \((k_{\hat{i}-1}^{(1)}, k_{\hat{i}}^{(1)}]\), which contains \(\hat{k}\) and has a maximum length of \(\lfloor \frac{K}{s} \rfloor\). After \(t\) successive applications of SL-SEED-MAX, the search interval is reduced to \((k_{\hat{i}-1}^{(t)}, k_{\hat{i}}^{(t)}]\), with a maximum length of \(\lfloor \frac{K}{s^t} \rfloor\). 
After \(\lfloor \log_s K \rfloor\) repeats of SL-SEED-MAX, the length of the search interval is reduced to ensure it contains only \(\hat{k}\). 

We refer to this iterative application of SL-SEED-MAX, as described in \eqref{eq:change point detection criterion}, as the \emph{Multi-Layer SEED-MAX (ML-SEED-MAX)} algorithm, or simply \emph{SEED-MAX} for brevity. The detailed implementation of the SEED-MAX algorithm is presented in Algorithm \ref{alg:cpd}.
We will show in the next section that with high probability, the index \(k^*\) of the most parsimonious family over the original class \(\{\mathcal{G}_k\}_{k=1}^K\) will reside within the final search interval produced by Algorithm \ref{alg:cpd}. Specifically, with high probability, we have \(k^* \in (k_{\hat{i}-1}^{(t)}, k_{\hat{i}}^{(t)}]\).

\begin{algorithm}
	\caption{Successive Empirical Excess Risk Difference Maximization (SEED-MAX) Algorithm}\label{alg:cpd}
	
	\begin{algorithmic}[1]
		\Require 
		Training data \(\mathcal{D} = \{(\boldsymbol{x}_n, \boldsymbol{y}_n)\}_{n=1}^N\); 
		Feature vector \(\boldsymbol{\phi}(\cdot, \boldsymbol{S})\) defined in \eqref{eq:feature vector}, where \(\boldsymbol{S}\) is a random sample matrix; 
		Number of segments \(s\); 
		Depth parameter \(t\).  
		\Ensure 
		Number of retained features \(k_{\hat{i}}^{(t)}\), which serves as a consistent estimate of \(k^*\) based on the sorting \(\boldsymbol{r}\) in \eqref{eq:emp sorting}; 
		Lower bound of the search interval after \(t\) iterations, denoted as \(k_{\hat{i}-1}^{(t)}\).
		
		\State Compute the sorting $\boldsymbol{r}$ according to \eqref{eq:emp sorting}
		\State  $l \leftarrow 1$, $u \leftarrow K$, $j \leftarrow t$ 
		
		\While{$j > 0$}
		\State Compute the sequence of indices $\{k_i\}_{i=1}^{s+1}$ by partitioning the discrete interval ${l,\dots, u}$ to $s$ almost equal parts in a way that $k_1 = l < k_2 < \dots < k_{s+1} = u$.  
		
		\State Solve the optimization in \eqref{eq:change point detection criterion} for sequentially nested class $\{\mathcal{G}_{k_{i}}\}_{i=1}^{s+1}$ to obtain $k_{\hat{i}}$. 
		\State $l \leftarrow k_{\hat{i}-1}$
		\State $u \leftarrow k_{\hat{i}}$
		\State $j \leftarrow j - 1$
		\EndWhile
		\State $k_{\hat{i}}^{(t)} \leftarrow u$
		\State $k_{\hat{i} - 1}^{(t)} \leftarrow l$
		\Statex \textbf{Outputs:} $\boldsymbol{r}$,  $k_{\hat{i}}^{(t)}$, and $k_{\hat{i} - 1}^{(t)}$.
	\end{algorithmic}
\end{algorithm}

\section{Theoretical Analysis of SEED-MAX}\label{sec:seedmax analysis}
In this section, we derive a lower bound for the probability of correctly estimating \(k^*\) using the SEED-MAX algorithm.

\begin{theorem}\label{thm:pac lower bound}
	Suppose that the classes of loss functions $\mathcal{L}_k = \{ l(\boldsymbol{g}(\cdot), \cdot) | \boldsymbol{g} \in \mathcal{G}_k\}$, with $k\in [K]$, over the class $\{\mathcal{G}_k \}_{k=1}^K$ are Glivenko-Cantelli, which means for every $\epsilon>0$, the relation \eqref{eq:glivenko-cantelli} holds. Moreover, we assume that $\{(k, \Delta R_k)\}_{k=1}^{k^* -1}$ is strictly discrete midpoint concave\footnote{From \cite{murota2003discrete}, for the strictly midpoint concave mappings it follows that $\Delta R_i + \Delta R_{j} < 2 \Delta R_{\lfloor\frac{i+j}{2}\rfloor}$ for all $i,j \in [k^* -1]$.}. Then for every $\delta \in (0,\frac{1}{24(s-2)t}]$, there exists a positive integer $N_0$ such that for every $N \ge N_0$ the following holds for the solution $k_{\hat{i}}^{(t)}$ from Algorithm \ref{alg:cpd}
	\begin{equation}\label{eq:pac cpd lower bound}
		\mathbb{P}\left[ k^* \in \left(k_{\hat{i}-1}^{(t)}, k_{\hat{i}}^{(t)}\right]\right] \ge 1 - 24 (s-2) t \delta,
	\end{equation}
	where $t$ and $s$ are the depth parameter and number of segments, respectively and $(k_{\hat{i}}^{(t)} - k_{\hat{i}-1}^{(t)}) < \frac{K}{s^t}$.
\end{theorem}
\begin{proof}
	See \ref{sec:main-thm-proof}. 
\end{proof}
We define \( a := \frac{K}{s^t} \) as the resolution parameter, which represents the length of the feature selection search interval after \( t \) successive applications of the SL-SEED-MAX search, as described in \eqref{eq:change point detection criterion}. Consequently, the number of iterations \( t \) can be rewritten as \( t = \log_s \left(\frac{K}{a}\right) \). With this parameterization, the PAC bound presented in \eqref{eq:pac cpd lower bound} can be reformulated as follows
\begin{align}\label{eq:pac bnd uncrtnty resol}
	\mathbb{P}\left[ k^* \in \left(k_{\hat{i}-1}^{(t)}, k_{\hat{i}}^{(t)}\right]\right] \ge 1 - 24 (s-2) \log_s \left(\frac{K}{a}\right) \delta.
\end{align}
From \eqref{eq:pac bnd uncrtnty resol}, it is evident that for a fixed number of segments \( s \) and a fixed uncertainty bound \( \delta \) (which, it can be shown, corresponds to a fixed number of training samples \( N \)), the lower bound for the probability of correctly localizing \( k^* \) decreases as the resolution parameter \( a \) becomes smaller, and vice versa. This suggests that achieving precise localization of \( k^* \) becomes more difficult at finer search resolutions compared to coarser ones.

Based on the Glivenko-Cantelli assumption, the bound uncertainty \(\delta\) in \eqref{eq:pac cpd lower bound} can be made arbitrarily small by increasing the training size \(N\). It can be shown that for ridge regression classification with a squared error loss function, \( \delta \) decreases exponentially with respect to \( N \) \cite{li2021towards}.

From \eqref{eq:emp sorting}, our proposed feature selection is based on ranking the regression weights. Since this ranking is derived from minimizing the empirical risk, misidentifying valuable features is possible. Furthermore, there may be correlations among the selected features, which can lead to the inclusion of redundant features. 
To identify these remaining redundant features, we examine the relationship between each feature and the label individually and filter out redundant features based on this analysis. This process will be addressed in the next section.

\section{Further Pruning via Fisher Test}\label{sec:Further pruning via Fisher Test}
The relevance of a feature to the labels can be quantified using mutual information, which measures their dependence. However, calculating the mutual information is often computationally demanding. Fisher test  \cite{fisher1970statistical}, \cite{feller1991introduction}, can be utilized as a more computationally efficient alternative. This test evaluates feature importance based on the ratio of the variance within classes to the variance between classes, providing a simpler yet effective measure of a feature's discriminative power.

Let \(\mathcal{N}_j := \{ n \in [N] \mid y_n = j \}\), with \(j \in [J]\), denote the set of training data indices whose labels belong to the \(j\)-th class. The cardinality of this set, \(N_j := \mathrm{card}(\mathcal{N}_j)\), is the number of training samples in the \(j\)-th class. Let \(\bar{\phi}_k^{(j)}\) represent the sample mean of the \(k\)-th feature for the \(j\)-th class, defined as \(\bar{\phi}_k^{(j)} := \frac{1}{N_j} \sum_{n\in \mathcal{N}_j} \phi(\boldsymbol{x}_n, \boldsymbol{\theta}_k)\). Additionally, let \(\bar{\phi}_k\) denote the mean of the \(k\)-th feature across all \(N\) training samples, defined as \(\bar{\phi}_k := \frac{1}{N} \sum_{n=1}^N \phi(\boldsymbol{x}_n, \boldsymbol{\theta}_k)\). We define \(v_1(k)\) as the inter-class variability of the \(k\)-th feature, given by \(v_1(k) = \sum_{j=1}^J N_j \left(\bar{\phi}_k^{(j)} - \bar{\phi}_k\right)^2\), and \(v_2(k)\) as the intra-class variability, given by \(v_2(k) = \sum_{j=1}^J\sum_{n\in \mathcal{N}_j} \left(\phi(\boldsymbol{x}_n, \boldsymbol{\theta}_k) - \bar{\phi}_k^{(j)}\right)^2\).

Intuitively, when the between-class variability  \(v_1(k)\) of a feature is large relative to its within-class variability \(v_2(k)\), the  $k$-th feature is likely to play a significant role in predicting the label. In contrast, features exhibiting low between-class variability tend to be less informative. Based on this principle, the Fisher test score for the $k$-th feature, denoted \(c_F(k)\), is defined in \cite{keshavarzian2024hiervar,fisher1921studies,fisher1970statistical}
\begin{align}\label{eq: fisher score}
	c_F(k) = \frac{(N - J)v_1(k)}{(J - 1)v_2(k)}.
\end{align}
Let \(\mathcal{J}_1\) be the index set of features retained by the SEED-MAX algorithm. We then retrain a subset of these SEED-MAX features whose Fisher scores exceed a specified threshold \(\tau\). Thus, the features retained after applying the Fisher test, i.e., those satisfying \( c_F(k) \geq \tau \), are identified as follows
\begin{align}\label{eq:fisher test}
	\mathcal{J}_2 = \left\{ k\in \mathcal{J}_1 \middle| c_F(k) \ge \tau \right\}. 
\end{align}
We refer to the feature selection method described in \eqref{eq:fisher test}, which is derived through a cascade application of SEED-MAX and the \emph{Fisher test (F-Test)}, as the \emph{SEED-MAX-F-Test} algorithm. 
Similar to the approach used in \cite{keshavarzian2024hiervar}, we define the threshold \(\tau\) as \(\tau = c_1 \mu\), where \(c_1 > 0\) is a universal constant (For the computation of suitable values of \( c_1 \), see Section \ref{sec:exp results}) and \(\mu\) represents the mean Fisher score of the features in the index set \(\mathcal{J}_1\), computed as follows
\begin{align}\label{eq:mu}
	\mu:= \frac{1}{\text{card}(\mathcal{J}_1)}\sum_{k \in \mathcal{J}_1} c_F(k).
\end{align}

In the next section, we present a numerical comparison of SEED-MAX and SEED-MAX-F-Test against several state-of-the-art feature-selection methods.

\section{Experimental Results}\label{sec:exp results}
In this section, we verify the performance of the analytical lower bound in \eqref{eq:pac cpd lower bound} from Theorem \ref{thm:pac lower bound}. We then numerically evaluate the performance of the SEED-MAX and SEED-MAX-F-Test algorithms on real public time-series datasets and compare them with other state-of-the-art methods\footnote{Additional comparisons of SEED-MAX and SEED-MAX-F-Test with state-of-the-art algorithms on synthetic autoregressive datasets are provided in Section 2 of the Supplementary Materials.}.

\subsection{Evaluation of the Analytical Lower Bound}
First, we examine the behaviour of the lower bound in \eqref{eq:pac cpd lower bound} for the probability that the most parsimonious model index $k^*$ lies within the estimated interval by SEED-MAX. To this end, we construct a dataset for a multiple-output regression task. This setting constitutes a special case of multi-task learning \cite{gillberg2016multiple}, with practical applications including computer emulation of mechanical systems \cite{fricker2011probabilistic}, prediction based on multivariate time series in sensor networks \cite{osborne2008towards}, and modeling interactions between multiple genes and a transcription factor protein \cite{lawrence2006modelling} (see Section 7 of \cite{alvarez2012kernels} for further discussion).  Accordingly, we generate a synthetic linear dataset as follows
\begin{equation}\label{eq:linear synthetic model}
	\boldsymbol{y} = \boldsymbol{W} \boldsymbol{x} + \boldsymbol{e},
\end{equation}
where $\boldsymbol{x} \in \mathbb{R}^{ K}$ is the input data, generated i.i.d. from a uniform distribution over the interval $[-1,1]$. The weight matrix $\boldsymbol{W}\in \mathbb{R}^{J \times K}$ is a sparse matrix, with only a subset by size $k^*$ of its column components being nonzero. The nonzero entries of $\boldsymbol{W}$ are generated i.i.d. from the normal distribution $\mathcal{N}(0,1)$ and are then fixed for each trial of the experiment.  The vector $\boldsymbol{e}\in \mathbb{R}^{J}$ represents the system error, whose entries are generated i.i.d. from the normal distribution $\mathcal{N}(0,\sigma^2)$, where $\sigma^2$ denotes the error variance. The vector $\boldsymbol{y}\in\mathbb{R}^{ J}$ denotes the noisy label vectors. 
We then apply SEED-MAX to the synthetic dataset $\mathcal{D}:=\{(\boldsymbol{x}_n,\boldsymbol{y}_n)\}_{n=1}^N$ constructed according to \eqref{eq:linear synthetic model}, and numerically compute the left-hand side of  \eqref{eq:pac cpd lower bound}.  

To examine the lower bound on the right-hand side of \eqref{eq:pac cpd lower bound}, it is necessary to establish an explicit relationship between the uncertainty parameter $\delta$ and the training dataset size $N$. To this end, we invoke the relation \eqref{eq:gc for seeds} from the proof of Theorem \ref{thm:pac lower bound}. Specifically, for every $\epsilon > 0$ and every $\delta \in (0,\frac{1}{12}]$, there exists a positive integer $N_0$ such that, for all $N\ge N_0$,  the relation \eqref{eq:gc for seeds} holds.

Since the underlying data-generating mechanism is known for the synthetic dataset, the left-hand side of \eqref{eq:gc for seeds} can be evaluated numerically via simulation for any choice of $\epsilon$. Moreover, $\epsilon$ is an arbitrary parameter. For a given value of $\epsilon$, we numerically estimate the probability $\mathbb{P}[| \Delta^2 \hat{R}_k - \Delta^2 R_k| \ge \epsilon]$. Given a training dataset of size $N$ and a specific choice of $\epsilon$, the parameter $\delta$ can then be set equal to its lower bound in \eqref{eq:gc for seeds}, namely $\frac{\mathbb{P}[| \Delta^2 \hat{R}_k - \Delta^2 R_k| > \epsilon]}{12}$. This value of $\delta$ is subsequently substituted into the lower bound on the right-hand side of \eqref{eq:pac cpd lower bound}. Repeating this procedure for different dataset sizes yields a numerical estimate of the lower bound, which can then be compared with the exact accuracy of estimating the size of the most parsimonious feature set.

The comparison between the numerical performance of SEED-MAX in identifying the most parsimonious model and the analytical lower bound in \eqref{eq:pac cpd lower bound} is illustrated in Fig. \ref{fig:sdmax analyitic lower}. These results were obtained with $K = 100$, $J = 10$, $s = 20$, $t = 3$, $k^* = 15$, and $\sigma = 0.1$. For each training dataset size, the experiment was repeated $9000$ times, and the average results are reported in Fig. \ref{fig:sdmax analyitic lower}. As observed, different choices of $\epsilon$ produce different lower bounds. However, when $\epsilon = 0.059$, the lower bound becomes relatively tight. This observation suggests that there exist suitable values of $\epsilon$ that can substantially tighten the lower bound and that such values can be identified in practice. Moreover, the lower bounds corresponding to other values of $\epsilon$ also increase with the training dataset size $N$. The results suggest that, as $N$ becomes sufficiently large, the differences among the lower bounds obtained for different choices of $\epsilon$ may become negligible. Both the exact estimation accuracy of $k^*$ and the corresponding lower bound (for $\epsilon = 0.059$) exhibit limited sensitivity to the training dataset size, and the lower bound closely tracks the behavior of the exact estimation accuracy. This synthetic experiment further demonstrates that SEED-MAX is an effective method for multiple-output regression tasks.

\begin{figure}
	\centering
	\includegraphics[width=0.6\linewidth]{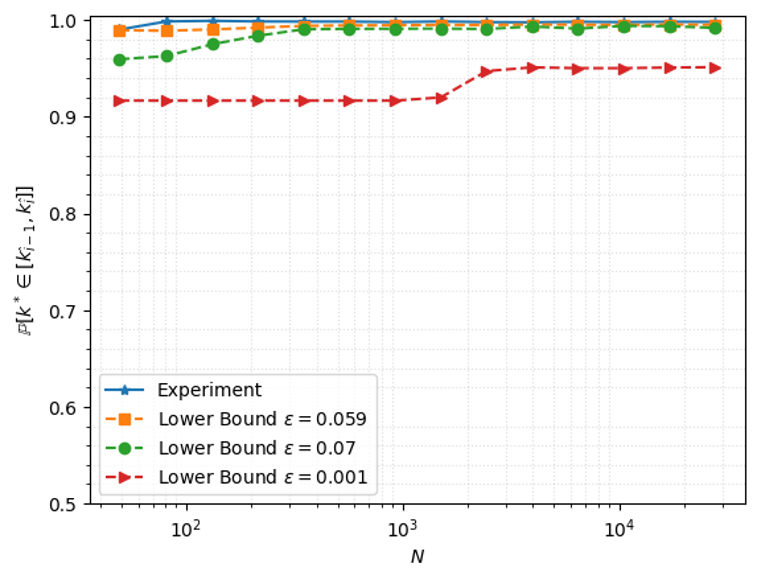}
	\caption{Comparing empirical accuracy of estimating the most parsimonious set of features with the lower bound of Theorem \ref{thm:pac lower bound}.}\label{fig:sdmax analyitic lower}
\end{figure}

\subsection{Evaluation over Real Datasets}
We leverage a specialized W-RKS feature extraction approach, \emph{minimally random convolutional kernel transform} (MINIROCKET) \cite{dempster2021minirocket}, designed specifically for time-series data. The SEED-MAX and SEED-MAX-F-Test algorithms are then applied to the extracted MINIROCKET features to eliminate redundancy. These methods are evaluated in comparison with other state-of-the-art feature selection techniques. We validate the proposed approach on 107 benchmark datasets from the UCR time-series classification archive\footnote{The UCR 2018 archive originally comprised 128 datasets. However, some datasets were excluded due to very low sample sizes per class (i.e., fewer than 10 input vectors per class). This refinement was undertaken to enhance the reliability and validity of the performance results computed over the UCR 2018 datasets.} \cite{dau2019ucr}.

The performance of each feature selection method for a given training dataset of size \( N \) is evaluated using three metrics: the probability of correct classification, \( P_{CC}(N) \), the rate of pruning, \( R_P(N) \), and the relative inference time\footnote{Here, $N$ is omitted from $T_I$, as $T_I$ is applied to the entire test set, which has a fixed size $N$.} $T_I$. The rate of pruning is defined as \( R_P(N) := 1 - \frac{\hat{k}(N)}{K} \), where \(\hat{k}(N)\) denotes the number of features retained by the feature selection algorithm with $N$ samples for training.  The relative inference time\footnote{Inference times are measured on an Intel Core i7 CPU without parallelization.} is defined as $T_I := \frac{T_{\hat{k}}}{T_K}$, where $T_{\hat{k}}$ is the inference time of the pruned model, and $T_K$ is the inference time of MINIROCKET.
Furthermore, for a set of training dataset sizes \( \mathcal{M} \subseteq \mathbb{N} \), the sequence \( \{(R_P(N), P_{CC}(N))\}_{N\in \mathcal{M}} \) is referred to as the \emph{performance trajectory} of the applied feature selection method.  
When the training sample size \( N \) is fixed, these metrics are denoted simply as \( P_{CC} \) and \( R_P \), respectively.

The hyperparameter settings for the MINIROCKET, SEED-MAX, and SEED-MAX-F-Test algorithms are outlined below.
Initially, the MINIROCKET random features are generated using the default configuration specified in \cite{dempster2021minirocket}, resulting $K =9996$ random features. 

For feature pruning using the SEED-MAX method, the parameter \( p \) in the sorting defined in \eqref{eq:emp sorting} is set to infinity, indicating the use of the maximum norm. Based on our empirical observations, this choice yields the best empirical performance.
Additionally, the depth parameter \( t \) and the number of segments \( s \) in the discrete search space are set to \( t = 3 \) and \( s = 10 \), respectively. This configuration results in a resolution parameter \( a \simeq 10 \), satisfying \( a \ll K \).

Additionally, consistent with the empirical findings in \cite{keshavarzian2024hiervar}, the hyperparameter of the F-Test in SEED-MAX-F-Test is set to \( c_1 = 0.5 \). All performance results presented are averaged over five trials of random feature extraction and pruning to ensure robustness and consistency.
In what follows, we present the experimental results over (practical) time-series classification datasets. We evaluated the SEED-MAX and SEED-MAX-F-Test methods on the UCR \cite{dau2019ucr} archive of time-series classification datasets.

Table \ref{tab:ucr-main} presents a comparison between SEED-MAX, SEED-MAX-F-Test, POCKET \cite{chen2024pocket}, Detach-ROCKET \cite{Uribarri2024DetachROCKET}, HIERVAR \cite{keshavarzian2024hiervar}, and F-Test \cite{fisher1921studies,fisher1970statistical,keshavarzian2024hiervar}. Across 107 UCR datasets, SEED-MAX-F-Test significantly reduces the number of retained features to 802, corresponding to a 92\% rate of pruning, with only a minimal classification accuracy degradation of approximately 1\%. The pruning rate of HIERVER and Detach-ROCKET are both lower than SEED-MAX-F-Test. Moreover, the classification accuracy of SEED-MAX-F-Test is 2 percent better than Detach-ROCKET. The inference time of SEED-MAX-F-Test is 6\% of the inference time of MINIROCKET. As a result, SEED-MAX-F-Test has the lowest inference time among all other state-of-the-art methods.

In the third column of Table \ref{tab:ucr-main}, the test accuracy and the number of retained features are presented for 13 out of 107 clean UCR datasets with ten or more classes. Notably, SEED-MAX-F-Test achieves the highest rate of pruning with minimal degradation in classification accuracy. Columns 4 to 6 of Table \ref{tab:ucr-main} detail the results for datasets with 10, 12, and 37 classes, respectively. As observed, the advantage of SEED-MAX-F-Test in achieving a high rate of pruning becomes increasingly pronounced as the number of classes grows, significantly outperforming other state-of-the-art methods presented in Table \ref{tab:ucr-main}.

In the fifth row of Table \ref{tab:ucr-main}, we present an ablation study evaluating the performance of \emph{Fisher’s test} (F-Test) when used as a standalone feature selection method. In this case, features are pruned using the F-score threshold proposed in \cite{keshavarzian2024hiervar}. As shown in Table \ref{tab:ucr-main}, Fisher’s test alone achieves relatively low pruning rates on the UCR datasets (at most 42 percent) and, consequently, does not constitute a state-of-the-art feature selection method when used in isolation.

\begin{table*}
	\caption{Empirical results over 107 cleaned UCR \cite{dau2019ucr} time-series classification datasets, evaluated using the probability of correct classification $(P_{CC})$, the rate of pruning $(R_P)$, and the relative inference time $(T_I)$.}
	\label{tab:ucr-main}
	\centering
	\tiny
	\setlength{\tabcolsep}{3pt}
	\begin{tabularx}{\textwidth}{|
			l|
			Y|Y|Y|
			Y|Y|Y|
			Y|Y|Y|
			Y|Y|Y|
			Y|Y|Y|}
		\hline
		
		\textbf{Method}
		& \multicolumn{3}{c|}{\textbf{UCR (107)}}
		& \multicolumn{3}{c|}{\textbf{UCR ($\ge 10$ classes, 13)}}
		& \multicolumn{3}{c|}{\textbf{UCR (10 classes, 6)}}
		& \multicolumn{3}{c|}{\textbf{UCR (12 classes, 5)}}
		& \multicolumn{3}{c|}{\textbf{UCR (37 classes, 1)}} \\
		
		\cline{2-16}
		
		& \scalebox{0.9}{$P_{CC}(\%)$} &$R_P(\%)$ & $T_I$
		& \scalebox{0.9}{$P_{CC}(\%)$} &$R_P(\%)$ & $T_I$
		&\scalebox{0.9}{$P_{CC}(\%)$} &$R_P(\%)$ & $T_I$
		&\scalebox{0.9}{$P_{CC}(\%)$} & $R_P(\%)$ & $T_I$
		& \scalebox{0.9}{$P_{CC}(\%)$} & $R_P(\%)$ & $T_I$ \\
		\hline
		
		\textbf{MINIROCKET}
		& \textbf{86} & 0 & 1
		& \textbf{85} & 0 & 1
		& \textbf{82} & 0 & 1
		& \textbf{72} & 0 & 1
		& 82 & 0  & 1\\
		\hline
		
		\textbf{POCKET}
		& 85 & 65 & 0.48
		& 82 & 57 & 0.57
		& \textbf{82} & 62 &0.35
		& \textbf{72} & 39 & 0.56
		& 82 & 65 & 0.29 \\
		\hline
		
		\textbf{Detach-ROCKET}
		& 83 & 88 & 0.11
		& 84 & 85 & 0.13
		& 80 & 78 & 0.22
		& 69 & 90 & 0.09
		& 81 & 82 & 0.13 \\
		\hline
		
		\textbf{HIERVAR}
		& 85 & 91 & 0.16
		& \textbf{85} & 74 & 0.24
		& \textbf{82} & 82 & 0.14
		& 71 & 79 & 0.18
		& \textbf{83} & 59 & 0.35 \\
		\hline
		
		\textbf{F-Test}
		& 85 & 32 & 0.9
		& \textbf{85} & 26 & 0.99
		& \textbf{82} & 29 & 0.7
		& \textbf{72} & 42 & 0.56
		& 82 & 20 & 0.8 \\
		\hline
		
		\textbf{SEED-MAX}
		& 85 & 86 & 0.11
		& 84 & 88 & 0.1
		& 80 & 88 & 0.09
		& 69 & 88 & 0.11
		& 82 & 88  & 0.08\\
		\hline
		
		\textbf{SEED-MAX-F-Test}
		& 85 & \textbf{92} & \textbf{0.06}
		& 84 & \textbf{91} & \textbf{0.07}
		& 81 & \textbf{91} & \textbf{0.06}
		& 70 & \textbf{92} & \textbf{0.07}
		& \textbf{83} & \textbf{90} & \textbf{0.07}\\
		\hline
		
	\end{tabularx}
%\begin{minipage}{\linewidth}
%			\textbf{Note:}
%			The numbers in parentheses in columns 2 to 6 indicates the number of datasets in that group.
%		\end{minipage}
\end{table*}

In Fig. \ref{fig:ucr perf plane all data}, the performance of various feature selection methods across all cleaned UCR datasets with \(J\)-classes has been compared with each other. Subfigures \ref{fig:subfig-2class}, \ref{fig:subfig-4class}, and \ref{fig:subfig-8class} correspond to cases in which the number of classes \(J\) equals 2, 4, and 8, respectively. Subfigure \ref{fig:subfig-12class}, illustrates the performance results aggregated over all datasets with \(J \ge 10\) classes. As the number of classes increases, the SEED-MAX-F-Test consistently achieves the highest rate of pruning without incurring significant degradation in classification accuracy.

In Fig. \ref{fig:ucr perf plane specific}, the performance trajectory of various feature selection methods has been illustrated on two specific UCR datasets: the FordA dataset with two classes and the MelbournePedestrian dataset with 10 classes. 
As observed, the classification accuracies of SEED-MAX and SEED-MAX-F-Test are comparable to those achieved by MINIROCKET. 
For datasets with a larger number of classes, such as MelbournePedestrian, the rate of pruning of SEED-MAX-F-Test surpass those of all other state-of-the-art methods, without observable degradation in performance (in terms of the probability of correct classification). %Additionally, applying the F-Test feature selection algorithm to the features retained by SEED-MAX further enhances the rate of pruning.
% Wide figure across two columns
\begin{figure}
	\centering
	
	% ---------------- First row ----------------
	\begin{subfigure}{0.48\textwidth}
		\centering
		\includegraphics[width=0.8\linewidth, height = 6 cm]{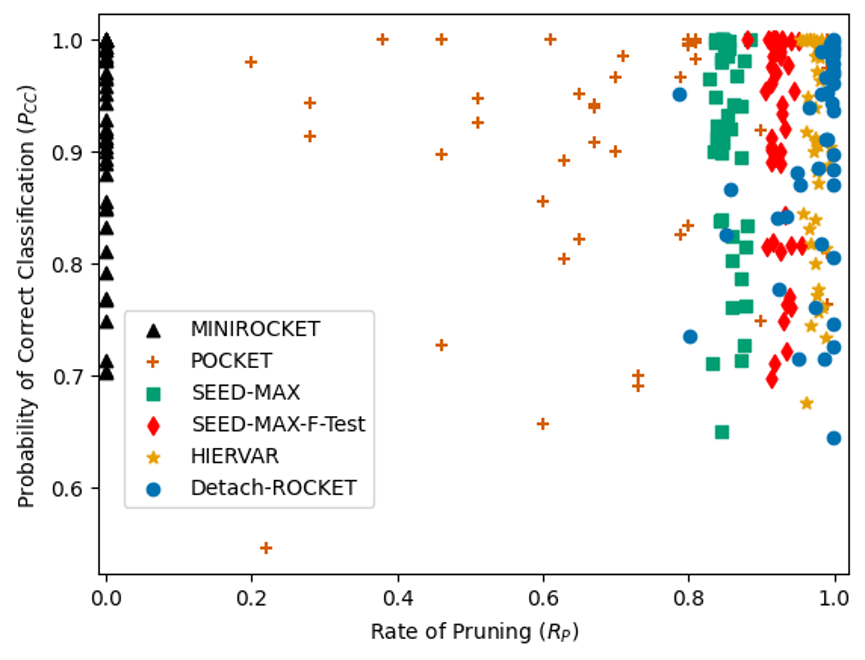}
		\caption{2-Class datasets}
		\label{fig:subfig-2class}
	\end{subfigure}
	\hfill
	\begin{subfigure}{0.48\textwidth}
		\centering
		\includegraphics[width=0.8\linewidth, height = 6 cm]{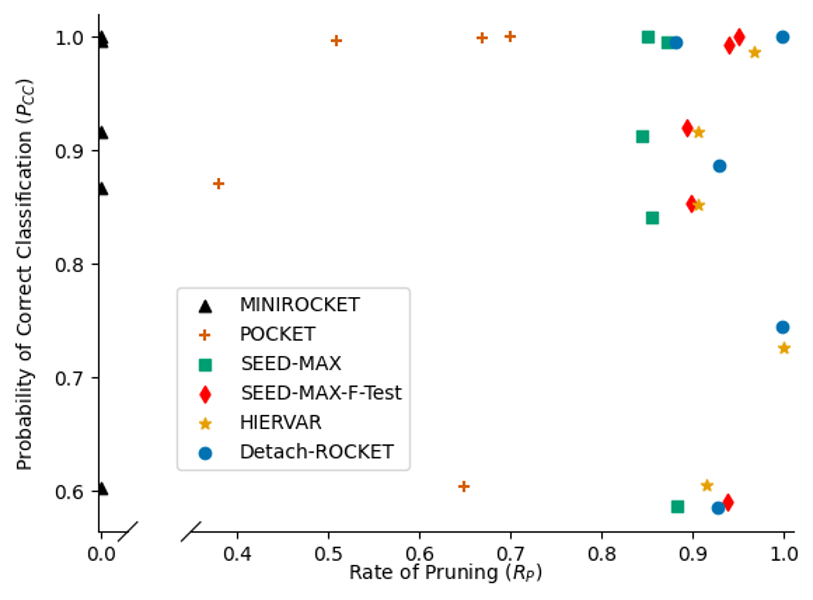}
		\caption{4-Class datasets}
		\label{fig:subfig-4class}
	\end{subfigure}
	
	\vspace{0.3cm}
	
	% ---------------- Second row ----------------
	\begin{subfigure}{0.48\textwidth}
		\centering
		\includegraphics[width=0.8\linewidth, height = 6 cm]{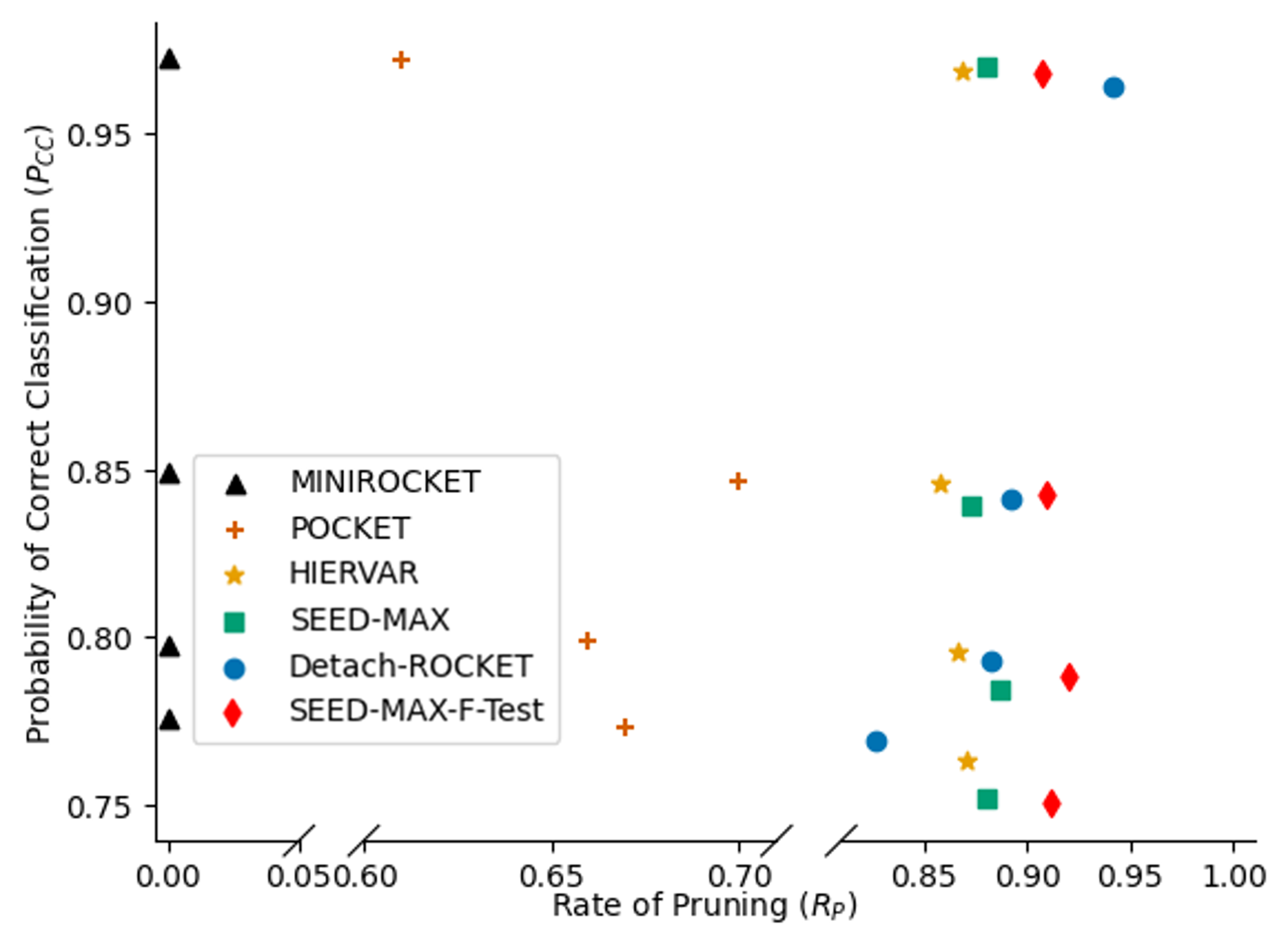}
		\caption{8-Class datasets}
		\label{fig:subfig-8class}
	\end{subfigure}
	\hfill
	\begin{subfigure}{0.48\textwidth}
		\centering
		\includegraphics[width=0.8\linewidth, height = 6 cm]{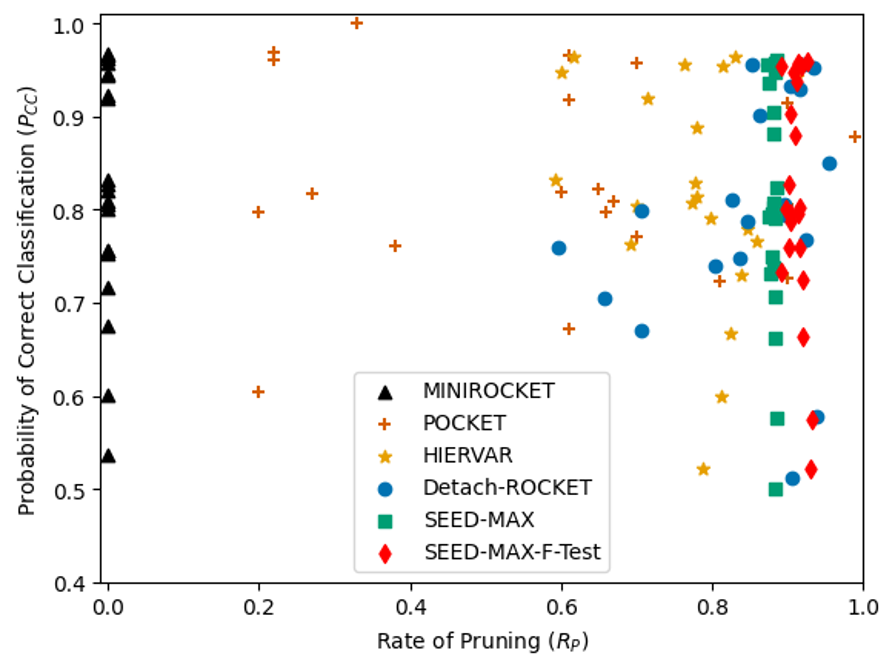}
		\caption{No. of classes: $\geq 10$}
		\label{fig:subfig-12class}
	\end{subfigure}
	
	\caption{Performance comparison of feature selection methods based on the probability of correct classification $P_{CC}$ versus the rate of pruning $R_p$ on UCR datasets with varying class counts.}
	\label{fig:ucr perf plane all data}
\end{figure}

\begin{figure}
	\centering
	\begin{subfigure}{0.4\textwidth}
		\centering
		
		\includegraphics[width=1\linewidth]{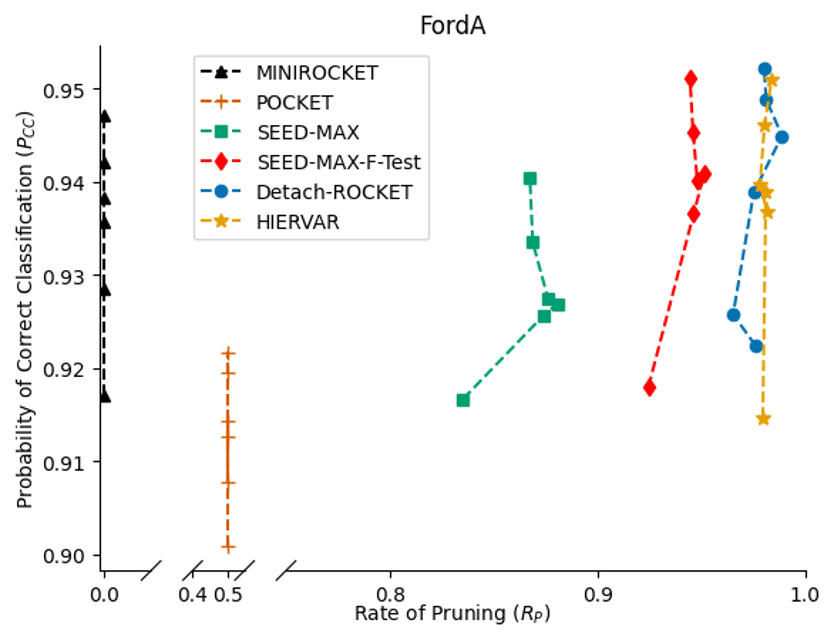}
		\caption{Real 2-Class data (FordA)}
	\end{subfigure}
	%\hfil
	\begin{subfigure}{0.4\textwidth}
		\centering

		\includegraphics[width=1\linewidth]{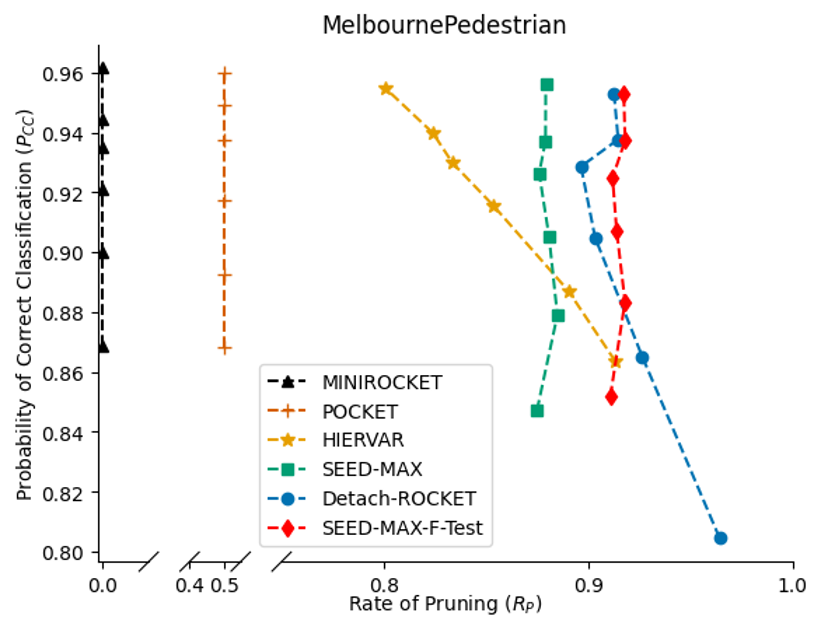}
		\caption{{Real 10-Class data (Melbourne Pedestrian).}}
	\end{subfigure}
	\caption{Performance trajectories of various feature selection methods on (a) FordA (in UCR archive) as a 2-Class dataset, over the set of different training sample sizes $\mathcal{M} =\{ 134,  258,  496,  954, 1830, 3508 \}$, and (b) MelbournePedestrian (in UCR archive) as a 10-Class dataset, over the set of different training sample sizes $\mathcal{M} =\{ 80,  130,  230,  400,  670, 1140\}$.  
	}\label{fig:ucr perf plane specific}
\end{figure}

\subsubsection{Critical Difference Diagram}\label{sec:cd diag}
In Table \ref{tab:ucr-main}, we report the average performance of different feature-selection methods across various subsets of the UCR datasets. However, the extent to which classification accuracies obtained from different domains are directly comparable, and whether averaging such accuracies across domains yields meaningful conclusions, remains a subject of debate in the machine learning literature \cite{webb2000multiboosting}. 
To address this, the Wilcoxon signed-rank test \cite{wilcoxon1945individual, demsar2006statistical} is recommended for reliably comparing multiple methods over diverse datasets. As the name implies, this test ranks the performance of each method on each dataset and then computes the mean rank across all datasets. It also determines whether the observed ranking constitutes a significant difference between methods. The results of this test are typically visualized using a \emph{critical difference diagram} \cite{demsar2006statistical}. We employ the Python implementation\footnote{\url{https://github.com/hfawaz/cd-diagram}} of the critical difference diagram described in \cite{IsmailFawaz2018deep} to compare state-of-the-art feature selection methods.

Fig. \ref{fig:cd-diag} compares the performance of five state-of-the-art feature selection methods against MINIROCKET, the baseline feature extraction method. In each subfigure, the methods are ranked from 1 to 6 according to the performance metric, either the probability of correct classification ($P_{CC}$) or the rate of pruning  ($R_P$). The average rank of each method appears next to its name and is marked on the horizontal axis; ranks closer to 1 indicate better performance. Thick horizontal lines connecting a group of methods on the axis signify that their ranking differences are statistically insignificant. With this context, we now interpret the graphs in Fig. \ref{fig:cd-diag}.

Fig. \ref{fig:cd-diag-pcc2} presents the performance of different feature selection methods on a subset of UCR datasets with binary classes, using probability of correct classification ($P_{CC}$) as the metric. The best rank (highest $P_{CC}$) belongs to MINIROCKET, while Detach-ROCKET receives the worst rank. However, the connecting thick horizontal lines indicate that these ranking differences are statistically insignificant, implying that all methods perform similarly in terms of $P_{CC}$. Fig. \ref{fig:cd-diag-rp2} shows the critical difference diagram for the rate of pruning ($R_P$) on the same binary-class subset. Here, the rankings differ significantly, as no thick horizontal line connects the methods. On binary UCR data, Detach-ROCKET and HIERVAR achieve higher rates of pruning than SEED-MAX and SEED-MAX-F-Test. The situation changes for multi-class datasets (subsets of UCR with at least 10 classes), as shown in Fig. \ref{fig:cd-diag-pcc-10} and Fig. \ref{fig:cd-diag-rp-10}. Again, in Fig. \ref{fig:cd-diag-pcc-10}, all methods exhibit comparable $P_{CC}$, but Fig. \ref{fig:cd-diag-rp-10} reveals that SEED-MAX-F-Test attains the best rate of pruning, followed by Detach-ROCKET and SEED-MAX with similar rates. Thus, on multi-class datasets, SEED-MAX-F-Test achieves a significantly higher rate of pruning without any significant degradation in $P_{CC}$.
\begin{figure}
	\centering
	\begin{subfigure}[b]{0.49\textwidth}
		\centering
		\includegraphics[width=\textwidth]{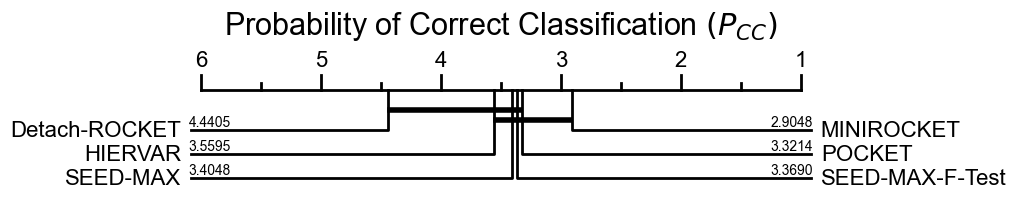}
		\caption{Ranking criterion: $P_{CC}$, No. of classes: 2.}
		\label{fig:cd-diag-pcc2}
	\end{subfigure}
	\hfill
	\begin{subfigure}[b]{0.49\textwidth}
		\centering
		\includegraphics[width=\textwidth]{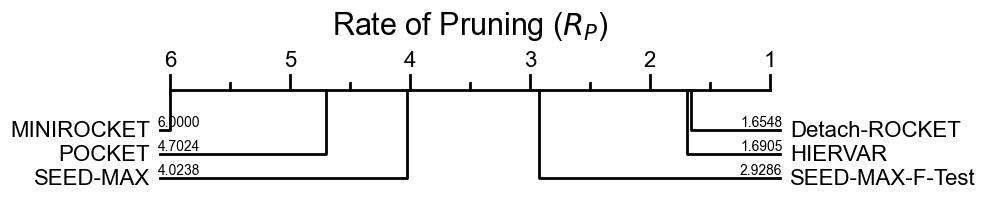}
		\caption{Ranking criterion: $R_P$, No. of classes: 2.}
		\label{fig:cd-diag-rp2}
	\end{subfigure}
	
	\vspace{1em} % Vertical space between rows
	
	\begin{subfigure}[b]{0.49\textwidth}
		\centering
		\includegraphics[width=\textwidth]{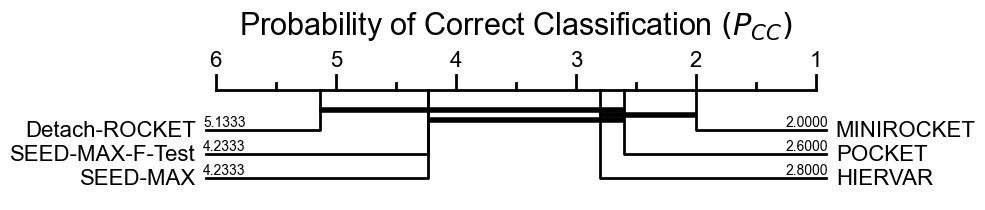}
		\caption{Ranking criterion: $P_{CC}$, No. of classes: $\ge 10$.}
		\label{fig:cd-diag-pcc-10}
	\end{subfigure}
	\hfill
	\begin{subfigure}[b]{0.49\textwidth}
		\centering
		\includegraphics[width=\textwidth]{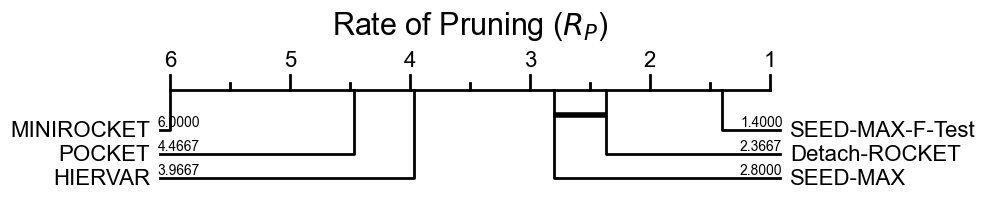}
		\caption{Ranking criterion: $R_P$, No. of classes: $\ge 10$.}
		\label{fig:cd-diag-rp-10}
	\end{subfigure}
	
	\caption{Critical difference diagrams for different feature selection methods. (a) Ranking based on $P_{CC}$ over 2-Class UCR Datasets. (b) Ranking based on $R_{P}$ over 2-Class UCR Datasets. (c) Ranking based on $P_{CC}$ over UCR Datasets with more than 10 classes. (d) Ranking based on $R_{P}$ over  UCR Datasets with more than 10 classes.}
	\label{fig:cd-diag}
\end{figure}

\section{Conclusion}\label{sec:conclusion}
In this paper, we proposed the SEED-MAX and SEED-MAX-F-Test feature selection methods for pruning redundant features generated by W-RKS. We derived a lower bound for the probability of correct feature selection by SEED-MAX and demonstrated theoretically that this bound can be made arbitrarily close to 1, albeit at the expense of increasing the training data size. Empirical evaluations conducted on UCR datasets showed that the SEED-MAX method, when combined with F-Test post-processing, achieved approximately 92\% feature pruning with less than 1\% degradation in classification accuracy. The rate of pruning advantage of our proposed methods becomes especially evident in multi-class classification problems involving many classes, where SEED-MAX-F-Test demonstrates superior performance compared to state-of-the-art approaches.

\appendix
\section{Proof of Theorem 2}\label{sec:main-thm-proof}
Before proceeding with the proof of Theorem \ref{thm:pac lower bound}, we first introduce and prove Lemma \ref{lem:seeds optimality for risks}, Lemma \ref{lem: gc for seeds}, Proposition \ref{prop:seeds tail}, and Proposition \ref{prop:sl loc garantee}.
We recall that the risk minimizer \( \boldsymbol{g}^*_{k} \) for the \( k \)-th family \( \mathcal{G}_k \) within the sequentially nested class \( \{ \mathcal{G}_k \}_{k=1}^K \) is defined as follows
\begin{align}\label{eq:risk minimizer reminder}
	\boldsymbol{g}^*_k = \operatorname{argmin}_{\boldsymbol{g}\in \mathcal{G}_k} R[\boldsymbol{g}].
\end{align}
Moreover, let us define first-order risk difference $\Delta R_k$ as follows
\begin{align}\label{eq:seer asymp reminder}
	\Delta R_k := R[\boldsymbol{g}^*_{k}] - R[\boldsymbol{g}^*_{k+1}].
\end{align}
Based on \( \Delta R_k \) in \eqref{eq:seer asymp reminder}, the quantity $\Delta^2R_k$ will be difined as follows
\begin{align}\label{eq:seed asymp reminder}
	\Delta^2 R_k := \Delta R_{k-1} - \Delta R_{k }.
\end{align}
Based on the above definitions, we state the following lemma.
\begin{lemma}\label{lem:seeds optimality for risks}
	Under the assumption that the mapping $\{(k,  \Delta R_k) \}_{k\in [k^* - 1]}$ is strictly discrete midpoint concave\footnote{From \cite{murota2003discrete}, for the strictly midpoint concave mappings it follows that $\Delta R_i + \Delta R_{j} < 2 \Delta R_{\lfloor\frac{i+j}{2}\rfloor}$ for all $i,j \in [k^*-1]$.}, it follows that 
	\begin{align}
		k^* = \operatorname{argmax}_{k\in [K]} \Delta^2R_k,
	\end{align}  
where $k^*$ is the index of the most parsimonious family of functions\footnote{See Theorem \ref{thm:elbow risk} for the formal definition of $k^*$.} among the $K$ sequentially nested families of functions $\{\mathcal{G}_k\}_{k = 1}^K$.
\end{lemma}
\begin{proof}
From the strictly discrete midpoint concavity of \(\{(k,\Delta R_k)\}_{k\in [k^* -1]}\), it follows that for every \( i, j \in [k^* - 1] \), the following holds
	\begin{align}\label{eq:discrete midpoint concavity cond}
		\Delta R_i + \Delta R_{j} < 2 \Delta R_{\lfloor\frac{i+j}{2}\rfloor}
	\end{align}
	Setting \( i = k - 1 \) and \( j = k + 1 \) in \eqref{eq:discrete midpoint concavity cond} implies that, for every \( k \in [k^* - 1] \), the following inequality holds
	\begin{align}\label{eq:concavity prop}
		\Delta^2 R_{k} < \Delta^2 R_{k+1}.
	\end{align}
	By \eqref{eq:concavity prop} and the transitivity of the relation \(<\), it follows that
	\begin{equation}
		\Delta^2 R_{1} < \Delta^2 R_{2} < \cdots < \Delta^2 R_{k^*},
	\end{equation}
	which implies
	\begin{align}\label{eq:discrete midpoint concavity assump}
		\Delta^2 R_{k} < \Delta^2 R_{k^*}, \quad \forall k \in [k^* - 1].
	\end{align}
	Moreover, from \eqref{eq:risk elbow property} and \eqref{eq:seer asymp reminder}, it follows that
	\begin{equation}\label{eq:deltarkm1}
		\Delta R_{k^* - 1} > 0,
	\end{equation}
and 
\begin{equation}\label{eq:deltark0}
	\Delta R_{k} = 0, \quad\forall k \in [K]\setminus[k^* - 1].
\end{equation}
	Consequently, using \eqref{eq:seed asymp reminder}, \eqref{eq:deltarkm1}, and \eqref{eq:deltark0}, we arrive at
	\begin{equation}\label{eq:delta2rkstar}
		\Delta^2 R_{k^*} = \Delta R_{k^* - 1} - \Delta R_{k^*} = \Delta R_{k^* - 1} > 0.
	\end{equation}
Moreover, from \eqref{eq:risk elbow property} and \eqref{eq:seed asymp reminder} it follows that
\begin{equation}\label{eq:zero after elbow}
	\Delta^2 R_{k} = 0, \quad \forall k\in [K]\setminus [k^*].
\end{equation}
	Finally, from \eqref{eq:discrete midpoint concavity assump}, \eqref{eq:delta2rkstar}, and \eqref{eq:zero after elbow} it follows that
	\begin{align}
		k^* = \operatorname{argmax}_{k\in [K]} \Delta^2 R_k,
	\end{align}
	which is the desired result.
	
\end{proof}

\begin{lemma}\label{lem: gc for seeds}
	Consider the sequentially nested class \( \{\mathcal{G}_k\}_{k=1}^K \), where each family of loss functions \(\mathcal{L}_k := \{ l(\boldsymbol{g}(\cdot),\cdot) | \boldsymbol{g} \in \mathcal{G}_k  \}  \), where $k \in [K]$, is Glivenko-Cantelli. Then, for every \( k \in [K] \), there exists a positive integer \( N_k \) such that for all \( N \geq N_k \), there exists a bound \( \epsilon_k(N, \delta) \) of order \( o(1) \), with \( \delta \in \left(0,\frac{1}{12}\right] \), ensuring that the following holds
	\begin{align}\label{eq:gc for seeds}
		\mathbb{P}\left[\left| \Delta^2 \hat{R}_k - \Delta^2 R_k\right| > \epsilon_k(N, \delta)\right] \le 12\delta.
	\end{align}
\end{lemma}
\begin{proof}
	Based on the definitions of $\Delta^2\hat{R}_{k}$ and $\Delta^2{R}_{k}$ in \eqref{eq:seeds} and \eqref{eq:seed asymp reminder}, respectively, we decompose $\Delta^2\hat{R}_{k} - \Delta^2{R}_{k}$  as
	\begin{equation}\label{eq:delta_2 R diff decompose}
			\Delta^2\hat{R}_{k} - \Delta^2{R}_{k} = d_{k-1} + d_{k+1} - 2 d_{k},
	\end{equation}
where 
\begin{equation}\label{eq:d_k}
	d_{k} := \hat{R}[\hat{\boldsymbol{g}}_{k}] - R[\boldsymbol{g}^*_{k}].
\end{equation}
	Furthermore, the random variable $d_k$ admits the decomposition
	\begin{align}\label{eq:atomic decomposition}
		d_k = \hat{R}[\hat{\boldsymbol{g}}_{k}] - R[\hat{\boldsymbol{g}}_{k}] + R[\hat{\boldsymbol{g}}_{k}] - R[\boldsymbol{g}^*_{k}].
	\end{align}
By the triangle inequality,
\begin{align}
	d_k \le |d_k| \le \left|\hat{R}[\hat{\boldsymbol{g}}_{k}] - R[\hat{\boldsymbol{g}}_{k}] \right| + \left|R[\hat{\boldsymbol{g}}_{k}] - R[\boldsymbol{g}^*_{k}]\right|
\end{align}
	From the Glivenko-Cantelli (GC) assumption, it follows that for every $\delta_k\in(0,1]$, there exist a positive integer $N_k$ and a bound $\epsilon_{k,1}(N,\delta_k)$ such that, for all $N\ge N_k$,
	\begin{align}\label{eq:gc for gk}
		\mathbb{P}\left[\sup_{\boldsymbol{g}\in \mathcal{G}_k}\left|\hat{R}[\boldsymbol{g}] - R[\boldsymbol{g}]\right| > \epsilon_{k,1}(N,\delta_{k})\right] \le \delta_{k}.
	\end{align}
Since $\hat{\boldsymbol{g}}_k,\boldsymbol{g}^*_k\in\mathcal{G}_k$, and 
\begin{equation}
\left| \hat{R}[\boldsymbol{g}] - R[\boldsymbol{g}] \right| \le \sup_{\boldsymbol{g}\in \mathcal{G}} \left| \hat{R}[\boldsymbol{g}] - R[\boldsymbol{g}]\right| 
\end{equation}
 holds for every $\boldsymbol{g}\in \mathcal{G}$, it follows from \eqref{eq:gc for gk} that
\begin{equation}\label{eq:gc g hat}
	\mathbb{P}\left[\left|\hat{R}[\hat{\boldsymbol{g}}_k] - R[\hat{\boldsymbol{g}}_k]\right| > \epsilon_{k,1}(N,\delta_{k})\right] \le \delta_{k},
\end{equation}
and
\begin{equation}\label{eq:gc g star}
	\mathbb{P}\left[ \left| \hat{R}[\boldsymbol{g}^*_k] - R[\boldsymbol{g}^*_k] \right| > \epsilon_{k,1}(N,\delta_{k}) \right] \le \delta_{k}. 
\end{equation}
Define
\begin{equation}\label{eq:xi}
	\xi :=  R[\hat{\boldsymbol{g}}_k] - \hat{R}[\hat{\boldsymbol{g}}_k] + \hat{R}[\boldsymbol{g}^*_k] - R[\boldsymbol{g}^*_k] .
\end{equation}
By the definitions of the population risk minimizer and the empirical risk minimizer, $R[\hat{\boldsymbol{g}}_k] \ge R[\boldsymbol{g}^*_k]$, we have $\hat{R}[\boldsymbol{g}^*_k] \ge \hat{R}[\hat{\boldsymbol{g}}_k]$. Applying these inequalities to \eqref{eq:xi} yields
\begin{equation}\label{eq:gstar best}
	0 \le R[\hat{\boldsymbol{g}}_k] - R[\boldsymbol{g}^*_k] \le \xi.
\end{equation}
Next, applying \eqref{eq:gc g hat} and \eqref{eq:gc g star} to \eqref{eq:xi}, together with the triangle inequality, gives
\begin{equation}\label{eq:merged r}
	\mathbb{P}\left[  \xi   > 2\epsilon_{k,1}(N,\delta_{k}) \right]  \le 2\delta_{k},
\end{equation}
for all $N\ge N_k$.
Using \eqref{eq:gstar best}, we further obtain
\begin{equation}\label{eq:ghat star gap}
	\mathbb{P}\left[  R[\hat{\boldsymbol{g}}_k] - R[\boldsymbol{g}^*_k]  > 2\epsilon_{k,1}(N,\delta_{k}) \right] \le \mathbb{P}\left[  \xi   > 2\epsilon_{k,1}(N,\delta_{k}) \right]  \le  2\delta_{k}.
\end{equation}
Applying \eqref{eq:gc g hat} together with \eqref{eq:ghat star gap}, and noting that $R[\hat{\boldsymbol{g}}_k] - R[\boldsymbol{g}^*_k] \ge 0$ always holds, yields
\begin{equation}\label{eq:erm gap}
	\mathbb{P}\left[\left|d_k \right| > 3 \epsilon_{k,1}(N,\delta_k) \right] \le 3 \delta_{k},
\end{equation}
where $d_k$ is defined in \eqref{eq:d_k}. 
Since \eqref{eq:erm gap} holds for every $k\in[K]$, it can also be applied to $d_{k-1}$ and $d_{k+1}$. Combining these bounds,  together with the decomposition in \eqref{eq:delta_2 R diff decompose}, and applying the triangle inequality,  yields
\begin{equation}\label{eq:seeds gap}
	\mathbb{P}\left[\left| \Delta^2 \hat{R}_k - \Delta^2 R_k\right| > 3 \epsilon_{k+1,1}(N,\delta_{k+1}) + 3 \epsilon_{k-1,1}(N,\delta_{k-1}) + 6 \epsilon_{k,1}(N,\delta_{k})\right] \le 3 (\delta_{k+1} +\delta_{k-1}) + 6\delta_{k}.
\end{equation}
Now define
\begin{equation}\label{eq:delta}
	\delta := \max\{\delta_{k-1}, \delta_{k}, \delta_{k+1}\},
\end{equation}
 and 
 \begin{equation}\label{eq:epsilonk}
 \epsilon_k(N,\delta) := 3 \epsilon_{k+1,1}(N,\delta_{k+1}) + 3 \epsilon_{k-1,1}(N,\delta_{k-1}) + 6 \epsilon_{k,1}(N,\delta_{k}).
 \end{equation}
By applying \eqref{eq:delta} and \eqref{eq:epsilonk} into \eqref{eq:seeds gap}, we obtain
\begin{equation}
	\mathbb{P}\left[\left| \Delta^2 \hat{R}_k - \Delta^2 R_k \right| > \epsilon_k(N,\delta) \right]\le 12 \delta,
\end{equation}
where $\delta \in (0,\frac{1}{12}]$.
\end{proof}

\begin{proposition}\label{prop:seeds tail}
	Let $k_{i^*}$ be the most parsimonious model index among the sorted model indices $\{k_i\}_{i=1}^{s+1}$ defined in \eqref{eq:sorted indices}.
Then, under the assumptions of Lemmas \ref{lem:seeds optimality for risks} and \ref{lem: gc for seeds}, for every $k\in [K]$, $k \neq k_{i^*}$ and every $\delta \in (0, \frac{1}{24}]$, there exists $N_0$ such that for all $N \ge N_0$,
\begin{equation}
	\mathbb{P}\left[ \Delta^2 \hat{R}_{k_{i^*}} < \Delta^2\hat{R}_k \right] \le 24 \delta.
\end{equation}
\end{proposition}
\begin{proof}
	Observe that
	\begin{align}\label{eq:delta2 R hat critical}
		\mathbb{P}\left[ \Delta^2 \hat{R}_{k_{i^*}} <  \Delta^2 \hat{R}_{k}\right] =\mathbb{P}\Big[ \Delta^2 \hat{R}_{k_{i^*}} - \Delta^2R_{k_{i^*}} +\Delta^2R_{k_{i^*}} - \Delta^2R_{k}  <  \Delta^2 \hat{R}_{k} - \Delta^2{R}_{k}\Big].
	\end{align}
As previously established in Lemma \ref{lem:seeds optimality for risks}, the quantity $k_{i^*}$ is the unique maximizer of $\Delta^2R_k$, it follows that 
\begin{equation}
	\Delta^2R_{k_{i^*}} - \Delta^2R_k > 0,
\end{equation} 
for every $k \in \{k_i\}_{i=1}^{s+1}\setminus \{k_{i^*}\}$.
Furthermore, Lemma \ref{lem: gc for seeds} implies that there exists and $N_1$ such that for every $N\ge N_1$ and every $\delta \in (0,\frac{1}{24}]$, the following holds
\begin{align}\label{eq:2delta term}
	\mathbb{P}\Big[\left|\Delta^2 \hat{R}_k - \Delta^2 R_k\right| + \left|\Delta^2 \hat{R}_{k_{i^*}} - \Delta^2 R_{k_{i^*}}\right| >  \epsilon_k(N,\delta) + \epsilon_{k_{i^*}}(N,\delta)\Big] \le 24\delta.
\end{align}
Since $\epsilon_k(N,\delta) \rightarrow 0$ as $N\rightarrow \infty$, there exists a positive integer $N_2$ such that for every $N \ge N2$,
\begin{align}\label{eq:epsilon below constant}
	\epsilon_k(N,\delta) + \epsilon_{k_{i^*}}(N,\delta) \le \Delta^2R_{k_{i^*}} - \Delta^2R_{k}.
\end{align}
Consequently,
\begin{align}\label{eq:delta bnd}
	\mathbb{P}\Big[&\left|  \Delta^2 \hat{R}_k -\Delta^2 R_k\right| + \left| \Delta^2 R_{k_{i^*}} - \Delta^2 \hat{R}_{k_{i^*}} \right|>  \Delta^2R_{k_{i^*}} - \Delta^2R_{k}\Big]\le 24\delta.
\end{align}
Applying the triangle inequality to \eqref{eq:delta bnd}, and simplifying the resulting expression yields
\begin{equation}
	\mathbb{P}\left[ \Delta^2 \hat{R}_{k_{i^*}} < \Delta^2\hat{R}_k \right] \le 24 \delta.
\end{equation}
This completes the proof.
\end{proof}
\begin{proposition}\label{prop:sl loc garantee}
	Let
	\begin{equation}\label{eq:cpd opt first layer}
		k_{\hat{i}} = \operatorname{argmax}_{k\in \{k_i\}_{i=2}^{s}} \Delta^2 \hat{R}_{k}.
	\end{equation}
Then, under the assumptions of Theorem \ref{thm:pac lower bound}, for every  $\delta \in \left(0, \frac{1}{24(s - 2)}\right]$, there exists a positive integer $N_0$ such that for all $N\ge N_0$,
\begin{align}\label{eq:seed sl lb}
	\mathbb{P}\left[k^* \in \left(k_{\hat{i}-1}, k_{\hat{i}}\right]\right] \ge 1 - 24(s - 2) \delta.
\end{align}
\end{proposition}
\begin{proof}
	By definition of $k_{\hat{i}}$,
	\begin{align}\label{eq: P k* in segment range}
		\mathbb{P}\left[k^* \in \left(k_{\hat{i}-1}, k_{\hat{i}}\right]\right] &=\mathbb{P}\left[\wedge_{k\in \{k_i\}_{i=2}^{s}\setminus \{k_{i^*}\}} \big(\Delta^2 \hat{R}_{k_{i^*}} \ge  \Delta^2 \hat{R}_{k}\big)\right] \notag\\
		&= 1 - \mathbb{P}\left[\vee_{k\in \{k_i\}_{i=2}^{s}\setminus \{k_{i^*}\}} \big(\Delta^2 \hat{R}_{k_{i^*}} <  \Delta^2 \hat{R}_{k}\big)\right].
	\end{align}
Applying the union bound gives
\begin{align}\label{eq:raw lower bound}
	\mathbb{P}\left[k^* \in \left(k_{\hat{i}-1}, k_{\hat{i}}\right]\right] \ge
	1 - 
	\sum_{k\in \{k_i\}_{i=2}^{s}\setminus \{k_{i^*}\}}  \mathbb{P}\left[ \Delta^2 \hat{R}_{k_{i^*}} <  \Delta^2 \hat{R}_{k}\right].
\end{align}
Applying Proposition \ref{prop:seeds tail} to the right-hand side of \eqref{eq:raw lower bound} yields
\begin{align}
	\mathbb{P}\left[k^* \in \left(k_{\hat{i}-1}, k_{\hat{i}}\right]\right] \ge 1 - 24(s - 2) \delta.
\end{align}
This completes the proof.
\end{proof}
Proposition \ref{prop:sl loc garantee} establishes the desired localization guarantee for a single depth of feature selection.
Now consider the multi-depth SEED-MAX procedure. Let define the events $\{E_j\}_{j\in [t]}$, where $t$ denotes the maximum depth number, as follows
\begin{equation}\label{eq:ej}
	E_j := \left\{k^* \in \left(k_{\hat{i}-1}^{(j)}, k_{\hat{i}}^{(j)}\right] \right\},
\end{equation}
where $\left(k_{\hat{i}-1}^{(j)}, k_{\hat{i}}^{(j)}\right]$ is the returned search interval by Algorithm \ref{alg:cpd} at the intermediate depth $j$. 
It is evident from Algorithm \ref{alg:cpd} that the sequence $\{E_j\}_{j=1}^{t}$ forms a Markov chain.
Conditioning on the interval obtained at depth $j-1$, Proposition \ref{prop:sl loc garantee} implies that for every $j\in [t]$,
\begin{equation}
	\mathbb{P}\left[ E_j \middle| E_{j-1} \right] \ge 1 - 24 (s-2) \delta,
\end{equation}
 where $E_0 :=\left\{ k^* \in [K]\right\}$ which is always true.
 Applying the chain rule of probability yields
 \begin{equation}\label{eq:et}
 	\mathbb{P}\left[ E_t \right] =  \prod_{j=1}^{t} \mathbb{P}\left[ E_j \middle| E_{j-1} \right].
 \end{equation}
 Applying Proposition \ref{prop:sl loc garantee} to the right-hand side of \eqref{eq:et} yields
 \begin{equation}
 		\mathbb{P}\left[ E_t \right]\ge \left( 1 - 24 (s-2)\delta \right)^t.
 \end{equation}
Finally, Bernoulli's inequality gives
\begin{equation}
	\left( 1 - 24 (s - 2) \delta \right)^t \ge 1 - 24 t (s-2) \delta. 
\end{equation}
Hence,
\begin{equation}
	\mathbb{P}\left[ k^* \in \left(k_{\hat{i}-1}^{(t)},k_{\hat{i}}^{(t)}\right] \right] \ge 1 - 24 t (s-2) \delta.
\end{equation}
This establishes the claim of Theorem \ref{thm:pac lower bound}.

\hfill \qed

%\section{Example Appendix Section}
%\label{app1}


%%       \cite{<label>} ==> [1]

%%
%Example citation, See \cite{lamport94}.

%\singlespacing
\bibliographystyle{elsarticle-num} 
\bibliography{bibtex/vswrm_bib}


%\begin{thebibliography}{00}

%% For numbered reference style

%\bibitem{lamport94}
 % Leslie Lamport,
  %\textit{\LaTeX: a document preparation system},
 % Addison Wesley, Massachusetts,
  %2nd edition,
  %1994.

%\end{thebibliography}
\end{document}